\documentclass[twoside]{article}

\usepackage[accepted]{aistats2026}
\usepackage[round]{natbib}

\usepackage[utf8]{inputenc} % allow utf-8 input
\usepackage[T1]{fontenc}    % use 8-bit T1 fonts
\usepackage{hyperref}       % hyperlinks
\usepackage{url}            % simple URL typesetting
\usepackage{booktabs}       % professional-quality tables
\usepackage{amsthm}         % better theorems
\usepackage{amsfonts}       % blackboard math symbols
\usepackage{nicefrac}       % compact symbols for 1/2, etc.
\usepackage{microtype}      % microtypography
\usepackage{xcolor}         % colors
\usepackage{multirow}
\usepackage{subcaption}

\usepackage[textsize=tiny]{todonotes}

\usepackage{algorithm}
\usepackage{algpseudocode}
\usepackage{tikz-cd}
\usepackage{amssymb}
\usepackage{amsmath,amsfonts,bm}

\def\eqref#1{equation~\ref{#1}}
\newtheorem{assumption}{Assumption}

\newtheorem{theorem}{Theorem}
\newtheorem{lemma}{Lemma}
\newtheorem{prop}{Proposition}

\newtheorem{cor}{Corollary}

\def\1{\bm{1}}

\DeclareMathAlphabet{\mathsfit}{\encodingdefault}{\sfdefault}{m}{sl}
\SetMathAlphabet{\mathsfit}{bold}{\encodingdefault}{\sfdefault}{bx}{n}

\def\gB{{\mathcal{B}}}

\def\gN{{\mathcal{N}}}
\def\gO{{\mathcal{O}}}

\def\gQ{{\mathcal{Q}}}

\def\gZ{{\mathcal{Z}}}

\newcommand{\E}{\mathbb{E}}

\newcommand{\R}{\mathbb{R}}

\DeclareMathOperator*{\argmax}{arg\,max}

\newcommand{\X}{\mathcal{X}} %V

\usepackage{bbm}

\def\Ac{\mathcal{A}}
\def\Cc{\mathcal{C}}

\def\Fc{\mathcal{F}}
\def\Rc{\mathcal{R}}
\def\Sc{\mathcal{S}}
\def\Zc{\mathcal{Z}}
\def\Hc{\mathcal{H}}

\def\E{\mathbb{E}}
\def\Rr{\mathbb{R}}

\begin{document}


%\runningauthor{Cioba \& Kayal et al.}
\twocolumn[

\aistatstitle{Reinforcement Learning Using Known Invariances}

%\aistatsauthor{ Alexandru Cioba$^{*1}$  \And Aya Kayal$^{*2}$  \And Laura Toni$^{3}$  \And Sattar Vakili$^{1}$ \And Alberto Bernacchia$^{1}$  }
%\aistatsauthor{
%Alexandru Cioba$^{*1}$ \And Aya Kayal$^{*2}$ \And Laura Toni$^{3}$\\
%Sattar Vakili$^{1}$ \And Alberto Bernacchia$^{1}$
%}
\aistatsauthor{
Alexandru Cioba$^{*1}$ \And
Aya Kayal$^{*2}$ \And
Laura Toni$^{3}$ \AND
Sattar Vakili$^{1}$ \And
Alberto Bernacchia$^{1}$
}

\begin{center}
\textsuperscript{*}Equal contribution
\end{center}
\aistatsaddress{$^{1}$MediaTek Research \And  $^{2}$American University of Beirut \And  $^{3}$University College London } ]

\begin{abstract}
In many real-world reinforcement learning (RL) problems, the environment exhibits inherent symmetries that can be exploited to improve learning efficiency. This paper develops a theoretical and algorithmic framework for incorporating known group symmetries into kernel-based RL. We propose a symmetry-aware variant of optimistic least-squares value iteration (LSVI), which leverages invariant kernels to encode invariance in both rewards and transition dynamics. Our analysis establishes new bounds on the maximum information gain and covering numbers for invariant RKHSs, explicitly quantifying the sample efficiency gains from symmetry. Empirical results on a customized Frozen Lake environment and a 2D placement design problem confirm the theoretical improvements, demonstrating that symmetry-aware RL achieves significantly better performance than their standard kernel counterparts. These findings highlight the value of structural priors in designing more sample-efficient reinforcement learning algorithms.
\end{abstract}

\section{INTRODUCTION}

Reinforcement Learning (RL) has achieved remarkable empirical success in diverse domains, including robotic manipulation, autonomous driving, game playing, and chip design~\citep{kalashnikov2018scalable,silver2016mastering,kahn2017uncertainty,mirhoseini2021graph,fawzi2022discovering}. However, despite these advances, the theoretical foundations of RL in complex environments—particularly those with continuous state-action spaces and nonlinear function approximation—remain less well understood. While recent work has established regret guarantees for tabular~\citep{jin2018q,auer2008near} and linear settings~\citep{jin2020provably,russo2019worst}, there is an increasing need to analyze RL algorithms that employ expressive function classes under realistic structural assumptions.

A key structure prevalent in real-world RL problems is symmetry. Many environments exhibit invariance under transformations like rotations, translations, or permutations. For instance:
Robotic arms often operate in rotationally symmetric workspaces;
Autonomous driving scenarios may be invariant to reflections or road mirroring;
Strategic games frequently contain symmetries due to interchangeable board configurations.
By explicitly incorporating these invariances into RL algorithms, we can enhance sample efficiency and generalization by eliminating redundant learning.

How can known environment symmetries be systematically exploited to enhance sample efficiency and regret guarantees under nonlinear function approximation?
To address this, we focus on kernel methods, which offer a theoretically tractable framework for modeling nonlinearity. Leveraging recent advances in invariant kernels for supervised and bandit learning~\citep{brown2024sample,haasdonk2007invariant,mei2021learning}, we introduce a rigorous theoretical and algorithmic framework for symmetry-aware RL via invariant Reproducing Kernel Hilbert Spaces (RKHSs). \textcolor{black}{In this work, we assume that the relevant invariances in the reward and transition dynamics are known a priori, as is common in many structured domains (e.g., robotics, board games, or resource allocation). Our approach is complementary to methods that learn invariances from data, such as meta-learning symmetry structures~\citep{zhou2020meta} or learning invariances via marginal-likelihood optimization in Gaussian processes~\citep{van2018learning}, and can naturally incorporate such learned structures when available.}

\paragraph{Contributions:} 
\begin{enumerate}
    \item We introduce a symmetry-aware variant of the kernel-based optimistic least-squares value iteration (LSVI) algorithm~\citep{yang2020provably}, incorporating group invariances in both reward and transition dynamics via a totally invariant kernel. 
    \item We provide a theoretical analysis of LSVI with totally invariant kernels, and give new bounds on the covering number of the function class of target $Q$-functions, which we translate into novel, symmetry-aware bounds on sample complexity. To our knowledge, this is the first theoretical analysis of the gains in sample complexity achieved in RL due to incorporating symmetries into the algorithmic design.  %Our theoretical analysis derives new bounds on the maximum information gain and covering number of invariant RKHSs, thereby explicitly quantifying the sample complexity gains afforded by symmetry. 
    \item Finally, we validate our theoretical findings through a series of experiments where natural geometric symmetry occurs in the MDP, demonstrating that incorporating prior knowledge of invariance into the kernel yields substantial gains in sample complexity.
\end{enumerate}

\section{RELATED WORK}
\label{sec:related_works}
\paragraph{Regret Bounds/Sample Complexities in RL}
Numerous studies have explored the sample complexity problem within the episodic MDP framework.  Both the tabular setting~\citep{jin2018q,auer2008near,bartlett2012regal} and the linear setting~\citep{jin2020provably,yao2014pseudo,russo2019worst,zanette2020frequentist,neu2020unifying} have been extensively investigated. In the tabular case,
optimistic state-action value learning algorithms~\citep{jin2018q} achieve a regret bound of $\gO\left(\sqrt{H^3 |\Sc \times \Ac| T}\right)$, where $H$ denotes the episode length, $T$ is the number of episodes, and $\Sc$ and $\Ac$ represent the finite state and action spaces, respectively. In the linear setting, regret bounds scale as $\gO\left(\sqrt{H^3 d^3 T}\right)$~\citep{jin2020provably}, depending on the dimension $d$ of the linear model rather than the size of the state-action space. Moreover, efforts have been made to extend these techniques to the kernelized setting~\citep{yang2020provably,yang2020reinforcement,chowdhury2019online,domingues2021kernel,vakili2023kernelized}, though further refinements are required to achieve tighter regret bounds. The most notable advancement in this direction is~\citet{yang2020provably}, which establishes regret guarantees for an optimistic least-squares value iteration (LSVI) algorithm. %, known as kernel optimistic least-squares value iteration (LSVI).
The regret bounds reported in~\citet{yang2020provably} are given by $ \gO\left(H^2\sqrt{(\Gamma(T)+\log \gN(\epsilon))\Gamma(T)T}\right)$, where $\Gamma(T)$ and $\gN(\epsilon)$ are kernel-related complexity measures—specifically, the maximum information gain and the $\epsilon$-covering number of the state-action value function class. In this work, we adopt the same kernel-based framework as~\citet{yang2020provably}, but for the first time, we incorporate group symmetries by introducing a totally invariant kernel. This allows us to exploit problem structure directly in the learning algorithm. We analyze the corresponding LSVI algorithm and derive new theoretical results by bounding $\Gamma(T)$ and $\gN(\epsilon)$ for invariant RKHSs, thereby establishing the first sample complexity guarantees for kernel-based RL methods that explicitly account for symmetry.

%Building on the regret bound by~\citet{yang2020provably}, we establish novel theoretical results on the sample complexity in the presence of symmetries by deriving explicit bounds on $\Gamma(T)$ and $\gN(\epsilon)$ for invariant RKHSs.

\paragraph{Invariant Kernel Methods}
Invariant kernel methods incorporate prior knowledge of invariances into kernel-based learning algorithms. The theoretical framework for group-invariant kernels was first introduced by~\cite{haasdonk2007invariant,kondor2008group}, and subsequent work has leveraged these ideas in Gaussian Processes (GPs), kernel ridge regression, and bandit optimization. For example,\cite{van2018learning} integrates invariances directly into the GP model structure rather than relying on data augmentation and develops a variational inference scheme to learn these invariances by maximizing the marginal likelihood. Other works, such as~\cite{mei2021learning} and~\cite{bietti2024sample}, explore kernelized regression with translation- and permutation-invariant kernels, respectively, analyzing the benefits of invariance in terms of improved sample complexity. A related study by~\cite{tahmasebi2023exact} provides an exact characterization of sample complexity gains when the target function is invariant under the action of an arbitrary Lie group, linking these gains to the dimensionality of the underlying group structure. In the context of kernelized bandit optimization,~\cite{brown2024sample} derive upper and lower bounds on sample complexity for Bayesian Optimization (BO) with invariant kernels when optimizing invariant objective functions. In this work, we extend these sample complexity results to the RL setting under symmetry assumptions on both the reward and transition dynamics.

\paragraph{RL with Symmetry}

%Symmetries in Markov Decision Processes (MDPs) have long been studied in connection with abstraction and model reduction.
\textcolor{black}{Classical work on symmetry in RL is closely connected to model minimization and MDP homomorphisms~\citep{ravindran2001symmetries,ravindran2002model,ravindran2004algebraic}, which formalize structure-preserving mappings between MDPs under which solving an abstract MDP yields an optimal policy for the original. More generally, state abstraction—mappings that reduce the state space while preserving selected reward or transition structure~\citep{ravindran2003smdp, li2006towards}—and bisimulation, which defines equivalence classes of states that are behaviorally indistinguishable~\citep{givan2003equivalence,taylor2008lax}, provide principled ways of exploiting structural equivalences in MDPs.}

Building on these ideas, symmetries in Markov Decision Processes (MDPs) have also been studied from a modern learning perspective. Several works apply data augmentation techniques in deep RL to improve sample efficiency and generalization on well-known benchmarks~\citep{yarats2021image, laskin2020reinforcement, lee2019network, cobbe2019quantifying, weissenbacher2022koopman, sinha2022s4rl}. Beyond data augmentation, a number of studies incorporate symmetry into neural network architectures to reduce sample complexity~\citep{van2020mdp, simm2020symmetry, wang2022mathrm, mondal2022eqr, nguyen2023equivariant, zhu2022sample}, including \citep{van2020mdp}, which designs equivariant networks that respect symmetry-induced homomorphisms and reduce the solution space.

These ideas have also been extended to the multi-agent RL setting, where symmetries can help reduce policy space complexity, enable the discovery of equivalent strategies, and facilitate coordination between agents. Such approaches often leverage graph-based methods or symmetry-aware architectures~\citep{van2021multi, liu2020pic, jiang2018graph, sukhbaatar2016learning, muglich2022equivariant, muglich2025expected, yu2024leveraging, tian2024exploiting, mcclellan2024boosting, yu2023esp, shi2025symmetry, chen2024rm}.

\textcolor{black}{In contrast to both classical abstraction methods and modern parametric approaches, our work focuses on single-agent RL and provides a non-parametric, kernel-based framework for incorporating known symmetries. Rather than learning abstractions or designing equivariant neural networks, we encode invariance directly into the function class via group-invariant kernels. Conceptually, the induced equivalence relation resembles the state aggregation underlying abstractions and homomorphisms, but is implemented implicitly through the kernel rather than an explicit mapping. This enables new theoretical insights: we analyze how invariance reduces the effective hypothesis space through tighter bounds on the maximum information gain and covering numbers of the invariant RKHS, thereby quantifying sample-efficiency gains. Finally, while modern deep RL symmetry methods such as data augmentation or equivariant architectures enforce invariances heuristically or parametrically, our non-parametric, kernel-based formulation provides formal guarantees of improved sample complexity.}

%In single-agent RL, symmetries in Markov Decision Processes (MDPs) have been formalized by~\citet{zinkevich2001symmetry, ravindran2001symmetries}. Several works have applied data augmentation techniques in deep RL to improve sample efficiency and generalization on well-known benchmarks~\citep{yarats2021image, laskin2020reinforcement, lee2019network, cobbe2019quantifying, weissenbacher2022koopman, sinha2022s4rl}. Beyond data augmentation, a number of studies incorporate symmetry into neural network architectures to reduce sample complexity~\citep{van2020mdp, simm2020symmetry, wang2022mathrm, mondal2022eqr, nguyen2023equivariant, zhu2022sample}, while others explore symmetry in the context of symmetric filters~\citep{clark2014teaching}.

%These ideas have also been extended to the multi-agent RL setting, where symmetries 
%can help reduce policy space complexity, enable the discovery of equivalent strategies, and facilitate coordination between agents. Such approaches often leverage graph-based methods or symmetry-aware architectures~\citep{van2021multi, liu2020pic, jiang2018graph, sukhbaatar2016learning, muglich2022equivariant, muglich2025expected, yu2024leveraging,tian2024exploiting,mcclellan2024boosting,yu2023esp,shi2025symmetry,chen2024rm}.

%However, our work differs significantly from both lines of research: we focus on single-agent RL in a kernel-based setting, which offers strong theoretical guarantees, and we encode symmetries directly into the learning algorithm via an invariant kernel.

\section{PRELIMINARIES AND PROBLEM FORMULATION}
\label{sec:prelim}

%In this section, we overview MDPs, symmetry assumptions, and kernel-based models.
In this section we introduce the problem formulation that we require for our theoretical results. Similar background can be found in~\citet{yang2020provably,vakili2023kernelized}. 
%[alex paper, yang paper, sattar paper, homomorphic paper]. 

\subsection{Episodic Markov Decision Processes}\label{sec:episodeMDP}

An episodic MDP is defined by the tuple $M = (\Sc, \Ac, H, P, r)$, where $\Sc$ and $\Ac$ are the state and action spaces, $H$ is the episode length, $r = \{r_h\}_{h=1}^H$ is the sequence of deterministic reward functions $r_h: \Zc \to [0,1]$, and $P = \{P_h\}_{h=1}^H$ is the sequence of transition distributions $P_h(\cdot \mid s,a)$ over next states, given $(s,a) \in \Zc = \Sc \times \Ac$. A policy $\pi = \{\pi_h\}_{h=1}^H$ maps states to actions at each step $h$, with $\pi_h: \Sc \to \Ac$. At the start of episode $t$, the environment provides $s_1^t$, and the agent selects policy $\pi^t$. At step $h$, the agent observes $s_h^t$, takes $a_h^t = \pi_h^t(s_h^t)$, receives $r_h(s_h^t, a_h^t)$, and transitions to $s_{h+1}^t \sim P_h(\cdot \mid s_h^t, a_h^t)$. The episode ends after $H$ steps.

The goal is to find a policy that maximizes the expected return from any state $s$ and step $h$, defined by the value function:
\begin{equation}
\begin{aligned}
V^{\pi}_h(s) &= \E\left[\sum_{h'=h}^H r_{h'}(s_{h'},a_{h'})\bigg|s_{h}=s\right], \\
&\quad \forall s\in\Sc,\; h\in[H].
\end{aligned}
\end{equation}
%\begin{equation}
%V^{\pi}_h(s) = \E\left[\sum_{h'=h}^Hr_{h'}(s_{h'},a_{h'})\bigg|s_{h}=s\right],  
%~~~\forall s\in\Sc, h\in[H],
%
%\end{equation}
where the expectation is over trajectories induced by $\pi$. Under standard assumptions, an optimal policy $\pi^\star$ exists such that $V_h^{\pi^\star}(s) = \max_\pi V_h^\pi(s)$ for all $s$ and $h$~\citep{puterman2014markov}. We write $V_h^\star(s) := V_h^{\pi^\star}(s)$, and define the expected value under $P_h$ as
\begin{equation}
[P_hV] := \E_{s'\sim P_h(\cdot|s,a)}[V_h(s')].
\end{equation}

The state-action value function $Q_h^\pi(s,a)$ is the expected return from $(s,a)$ at step $h$, defined as $Q_h^\pi(s,a) = \E_\pi\left[\sum_{h'=h}^H r_{h'}(s_{h'},a_{h'}) \mid s_h = s, a_h = a \right]$. The Bellman equation for $\pi$ is $Q_h^\pi(s,a) = r_h(s,a) + P_h V_{h+1}^\pi$, with $V_h^\pi(s) = \E_{a \sim \pi_h(s)}[Q_h^\pi(s,a)]$ and $V_{H+1}^\pi := 0$.
The Bellman optimality equations are $Q_h^\star(s,a) = r_h(s,a) + P_h V_{h+1}^\star$, and $V_h^\star(s) = \max_a Q_h^\star(s,a)$, with $V_{H+1}^\star := 0$.
The regret of a policy sequence $\{\pi^t\}_{t=1}^T$ is defined by:
\begin{equation}
\Rc(T) = \sum_{t=1}^T\left(V_1^\star(s_1^t) - V_1^{\pi^t}(s_1^t)\right),
\end{equation}
where $s_1^t$ is the initial state at episode $t$, and measures the cumulative value loss relative to the optimal policy.

% The state-action value function $Q_h^\pi(s,a)$ is defined as:
% \begin{equation}
% Q_h^\pi(s,a) = \E_\pi\left[\sum_{h'=h}^Hr_{h'}(s_{h'},a_{h'})\bigg|s_h=s, a_h=a\right].
% \end{equation}

% The Bellman equation for $\pi$ becomes:
% \begin{equation}
% Q_h^\pi(s,a) = r_h(s,a) + P_hV_{h+1}^\pi, \quad V_h^\pi(s) = \E_{a\sim\pi_h(s)}[Q_h^\pi(s,a)], \quad V_{H+1}^\pi := 0.
% \end{equation}

% The Bellman optimality equations are:
% \begin{equation}
% Q_h^\star(s,a) = r_h(s,a) + P_hV_{h+1}^\star, \quad V_h^\star(s) = \max_{a}Q_h^\star(s,a), \quad V_{H+1}^\star := 0.
% \end{equation}

% The regret of a policy sequence ${\pi^t}$ is defined by:
% \begin{equation}
% \Rc(T) = \sum_{t=1}^T\left(V_1^\star(s_1^t) - V_1^{\pi^t}(s_1^t)\right),
% \end{equation}
% where $s_1^t$ is the initial state at episode $t$. This quantity measures the cumulative loss in value due to not following the optimal policy.

\subsection{Symmetric RL}\label{sec:symmetric_RL}

In many real-world RL problems, the environment exhibits symmetries: structured transformations of states and actions that leave the reward and dynamics unchanged. In this section, we formalize this notion and provide intuitions for how symmetry can be leveraged to reduce sample complexity. We defer detailed derivations to Appendix~\ref{appx:grp_action}.

\paragraph{Group Actions on MDPs.} Let $(G, \cdot)$ be a group acting on a set $\X$.
%Let $(G, \cdot)$ be a group acting on a space $\X$ via a binary operation $\circ: G \times \X \to \X$ such that:
% \begin{itemize}
% \item $g \circ (h \circ x) = (g \cdot h) \circ x$ for all $g,h \in G$, $x \in \X$
% \item $e \circ x = x$ for identity element $e \in G$
% \end{itemize}
We denote the action map as $\circ: G \times \X \to \X$ and the induced action by $G \curvearrowright \X$. In the context of an MDP $\mathcal{M} = (\Sc, \Ac,H, P, r)$, we say $G$ acts on $\mathcal{M}$ if there are actions $G \curvearrowright \Sc$ and $G \curvearrowright \Ac$ such that, for all $s,s' \in \mathcal S$ and $a\in \mathcal A$:
\begin{align}
        \mathbb P(s'|g \circ_{\mathcal S} s, g \circ_{\mathcal A} a) &= \mathbb P(g^{-1} \circ_{\mathcal S} s'|s,a)\\
        % r(g\circ_{\mathcal S} s, g \circ_{\mathcal A}a, s') &= r(s, a, g^{-1} \circ_{\mathcal S} s')
        r(g\circ_{\mathcal S} s, g \circ_{\mathcal A}a) &= r(s, a)
    \end{align}
% \begin{align}
% P(s' | g \circ s, g \circ a) &= P(g^{-1} \circ s' | s, a), \\
% r(g \circ s, g \circ a) &= r(s, a).
% \end{align}
%When quantifying the gains in sample complexity of an algorithm due to the presence of symmetry, we are referring to the \emph{effective symmetries observed} from the point of view of $\mathcal X$. This corresponds to the image of $\phi$, $Im(\phi) \leq Aut(\mathcal X)$. Thus, though our actions are not required to be \textit{faithful}, in practice, this shall be the assumption throughout. 
The $\circ_{\mathcal S}$ and $\circ_{\mathcal A}$ actions allow us to define a coupling between the state and actions spaces. Functions that preserve this coupling are called \textit{equivariant}. In particular this applies to policies.
% \begin{defn}
\paragraph{Equivariant Policies.}  A policy $\pi : \mathcal S \to \mathcal P(\Ac)$ is called equivariant
%\footnote{Note that this differs from standard equivariance with respect to the actions $\circ_\Ac$ and $\circ_\Sc$. The action of $G$ on $\mathcal P(\Ac)$ by $g(a) = g^{-1}\circ_\Ac a$ is no longer a left action, but a right action.} 
with respect to $G$ if, for all $g$ and $s$, $\pi(g \circ_\Sc s) = g^{-1} \circ_\Ac \pi(s)$.
% , i.e. if the following diagram commutes:
%     \[
% \hspace{-.2\textwidth}\begin{tikzcd}
% \mathcal S \arrow[r, "\pi"] \arrow[d, "g"]
% & \mathcal P(\mathcal A) \arrow[d, "g^{-1}"] \\
% \mathcal S \arrow[r, "\pi" ]
% & \mathcal P (\mathcal A)
% \end{tikzcd}\hspace{-.2\textwidth}
% \]
% \end{defn}
% so, if:
% \begin{align}
%     \pi(gs) = g^{-1}\pi(s)
% \end{align}
% for all $g$ and $s$, where the binary operations are implicit to the space. 
% This captures the idea that transitions and rewards are preserved under transformations from $G$.
We are particularly interested in cases where such symmetries are known and can be used to constrain the hypothesis space of the learning algorithm. The following proposition which follows from first principles (see Appendix \ref{appendix:symrl} for derivation), justifies restricting solutions to the space of invariant, respectively equivariant functions for all policy-value iteration (PVI) algorithms. 
% \paragraph{Equivariant Policies.} A policy $\pi: \Sc \to \Ac$ is called \emph{equivariant} if:
% \begin{equation}
% \pi(g \circ s) = g^{-1} \circ \pi(s),
% \end{equation}
% for all $g \in G$ and $s \in \Sc$. This definition ensures that acting in a transformed state leads to a correspondingly transformed action.
\begin{prop}\label{prop:equivinv}
In finite-horizon MDPs, equivariant policies have invariant value functions $V_\pi(s)$ and $Q_\pi(s,a)$. Moreover, if $Q(s,a)$ is invariant under $G$ and satisfies the Bellman equation, then the greedy policy $\pi(s) = \arg\max_a Q(s,a)$ is equivariant.
\end{prop}
This result implies that if we start PVI in the space of invariant functions, and we incur no approximation error in either policy evaluation or iteration, we are guaranteed to recover an optimum, and such optimum is equivariant. 
However, in practical settings these guarantees don't apply. The following two results, which we present intuitively here, are meant to bridge the gap to the realistic setting where function approximation is of importance. 
\paragraph{Effective Reduction in the State Space}The first simply states that in some cases symmetries can be factored cleanly out of the MDP, resulting in a smaller, simpler space where all algorithms benefit from reduced sample complexity guarantees. 
If the action on $\Ac$ is trivial ($g \circ a = a$), we can define a reduced MDP over state orbits $\Sc/G$. This induces a simpler equivalent problem whose optimal policy can be lifted back to the original space.
%Even when the action is not trivial, incorporating knowledge of orbits can significantly reduce the sample complexity of learning. 
See Proposition \ref{prop:quotient_MDP} in Appendix~\ref{appx:prop:quotient_MDP}.

\paragraph{Simple Sample Complexity Bounds for $Q$-learning} The second is a variation on a celebrated result from \cite{jin2018q}, where we explore how sample complexity varies in $Q$-learning in the tabular setting under data augmentation. In tabular Q-learning, regret bounds typically scale with $|\Sc \times \Ac| = SA$. In the presence of symmetry, the relevant complexity becomes $\widetilde{SA} = |(\Sc \times \Ac)/G|$, the number of orbits. This motivates algorithmic design that explicitly incorporates symmetry, as we do in the rest of the paper. See Appendix~\ref{appx:tabular_Q} where we state and prove Theorem \ref{the:hoeffding}.

% \todo[size=\scriptsize,color=red!20!white]{\textbf{Alex}: Please provide either proofs of these propositions or references to other paper or appendix.} 

%Let $\text{Orb}(s)$ denote the orbit of $s$ under $G$, i.e., the set of all $g \circ s$.
% If the action on $\Ac$ is trivial ($g \circ a = a$), we can define a reduced MDP over state orbits $\Sc/G$. This induces a simpler equivalent problem whose optimal policy can be lifted back to the original space.
% %Even when the action is not trivial, incorporating knowledge of orbits can significantly reduce the sample complexity of learning. 
% See Proposition \ref{prop:quotient_MDP} in Appendix~\ref{appx:prop:quotient_MDP}.

%\todo[size=\scriptsize,color=red!20!white]{\textbf{Alex}: Please clarify the last two paragraph of section 3.2, for example by compressing them into one paragraph. Why do we talk about the trivial case? Instead, let's give a better intuition on why complexity reduces.} 

\subsection{Kernel-Based RL with Symmetries}\label{sec:kernel_RL_with_sym}

Kernel-based models are powerful tools for nonparametric function approximation in RL, offering flexible predictors and principled uncertainty estimates. Importantly, as seen in \cite{brown2024sample}, the pairing of the class of inner product kernels with orthogonal symmetries creates a powerful framework for incorporating symmetries naturally into predictors. %they naturally incorporate symmetry, making them well-suited for environments with known invariances.
See Appendix \ref{appendix:kernel} for standard definitions in kernel methods.

% Let $k: \Zc \times \Zc \rightarrow \R$ be a positive definite kernel. Its associated reproducing kernel Hilbert space (RKHS), $\Hc_k$, consists of functions $f: \Zc \to \R$ with inner product $\langle \cdot, \cdot \rangle_{\Hc_k}$ and norm $|\cdot|_{\Hc_k}$, satisfying the reproducing property $f(z) = \langle f, k(\cdot, z) \rangle_{\Hc_k}$.
% By Mercer’s theorem, $k$ admits an eigenfunction expansion
% \begin{equation}
% k(z, z') = \sum_{m=1}^{\infty} \lambda_m \varphi_m(z) \varphi_m(z'),
% \end{equation}
% where $\{\lambda_m\}_{m=1}^\infty$ are positive eigenvalues and $\{\varphi_m\}_{m=1}^\infty$ are orthonormal eigenfunctions in $\Hc_k$. Any $f \in \Hc_k$ can be expressed as $f = \sum_{m=1}^\infty w_m \sqrt{\lambda_m} \varphi_m$, with norm $\|f\|_{\Hc_k}^2 = \sum_{m=1}^\infty w_m^2$.

% \paragraph{Kernel-Based Prediction.} For $f \in \Hc_k$ and data ${(z_i, y_i)}_{i=1}^t$ with $y_i = f(z_i) + \varepsilon_i$ and zero-mean noise, kernel ridge regression yields posterior mean and variance:
% \begin{align}
% \hat{f}_t(z) = k_t^\top(z) (K_t + \rho I)^{-1} \bm{y}_t, \quad
% \sigma_t^2(z) &= k(z,z) - k_t^\top(z) (K_t + \rho I)^{-1} k_t(z),
% \end{align}
% where $k_t(z) = [k(z, z_1), \dots, k(z, z_t)]^\top$, $K_t = [k(z_i, z_j)]_{i,j=1}^t$, and $\rho > 0$ is a regularization parameter. Confidence bounds of the form $|f(z) - \hat{f}_t(z)| \leq \beta_t \sigma_t(z)$ hold with high probability under standard conditions.
%
\paragraph{Invariant Kernels.} Let $k: \Zc \times \Zc \rightarrow \R$ be a positive definite inner product kernel. Assume $\Zc \hookrightarrow \R^d$ is an embedding and the orthogonal group $O(d)$ restricts its action on $\Zc$.  Let $G$ be a finite subgroup of $O(d)$. Then the following formula defines an invariant kernel: 
\begin{equation}\label{invariant_kernel}
k_G(z, z') = \frac{1}{|G|} \sum_{g \in G} k(g(z), z'),
\end{equation}
The RKHS $\Hc_{k_G}$ induced by $k_G$ consists of $G$-invariant functions.
% Importantly, if $f \in \Hc_k$, the symmetrized function
% \begin{equation}
% S_G(f) := \frac{1}{|G|} \sum_{g \in G} f \circ g
% \end{equation}
% lies in $\Hc_{k_G}$, providing a constructive projection onto the invariant RKHS.
% %
Invariant kernels constrain learning to a smaller, structured function class, offering both theoretical and practical benefits in symmetric environments. In RL, they yield value and transition models that respect MDP symmetries, enhancing generalization and sample efficiency.
\subsection{Main Technical Assumptions}\label{sec:kernel_model_RL}
In our RL setting, we use a kernel-based model to predict the expected value function. For a given transition distribution $P(s'|\cdot,\cdot)$ and value function $v:\Sc \rightarrow \R$, we define $f = [P v]$ and use past observations to construct predictions and uncertainty estimates for $f$ via kernel ridge regression. Since the value functions evolve across steps due to the Markovian structure, we aim to control the estimation error at each step through confidence bounds. We will need the following standard assumption.
\begin{assumption}\label{ass:invRKHS_norm}
We assume $r_h(\cdot, \cdot), P_h(s'|\cdot, \cdot) \in \Hc_{k_G}$ for some invariant kernel $k_G$, and that $\|r_h\|_{\Hc_{k_G}}, \|P_h(s'|\cdot,\cdot)\|_{\Hc_{k_G}} \le 1$ for all $s' \in \Sc$ and $h \in [H]$.
\end{assumption}
Our analysis depends on the smoothness of the kernel $k$, which influences the function approximation properties of the associated RKHS. Specifically, we characterize the complexity of the kernel class through the decay of its Mercer eigenvalues.
\begin{assumption}[Eigendecay Profile]\label{ass:eigendecay}
Let $\{\lambda_m\}_{m=1}^\infty$ denote the Mercer eigenvalues of the invariant kernel $k_G$, ordered in decreasing magnitude. We assume the eigenvalues satisfy the decay condition $\lambda_m = \mathcal{O}((m |G|)^{-p})$ for some $p > 0$.
\end{assumption}
This assumption is mild and satisfied by many commonly used kernels. For instance, in the case of the Matérn kernel with smoothness parameter $\nu$ in $d$ dimensions, we have $p = 1 + \frac{2\nu}{d}$.
\subsection{LSVI Algorithm}\label{sec:lsvi}
We consider the Kernel-based Optimistic Value Iteration (KOVI) algorithm (\cite{yang2020provably}, see also Algorithm \ref{alg:KRVI}),  which applies optimistic LSVI in the kernel regression setting.
Least-Squares Value Iteration (LSVI) estimates $\widehat{Q}_h^t$ for the optimal $Q_h^\star$ at each step $h$ of episode $t$ via recursive Bellman updates. To encourage exploration, an upper confidence bonus $b_h^t: \Zc \to \R$ is added:
\begin{equation}
    Q_h^t = \min\{\widehat{Q}_h^t + \beta b_h^t,~ H - h + 1\}^+.
\end{equation}
Here, \( \beta > 0 \) is a tunable parameter, and the term \( \widehat{Q}_h^t + \beta b_h^t \) provides an optimistic estimate based on the principle of \emph{optimism in the face of uncertainty}. Specifically, the function \( \widehat{Q}_h^t \) is the kernel-based predictor for \( r_h + [P_h V^t_{h+1}] \) using the dataset \( \{r_h(z_h^{\tau}) + V^t_{h+1}(s_{h+1}^\tau)\}_{\tau=1}^{t-1} \) observed at inputs \( \{z_h^{\tau}\}_{\tau=1}^{t-1} \). Since rewards are bounded by $1$, the cumulative return from step $h$ is at most $H - h + 1$. At episode $t$, the agent follows a greedy policy $\pi^t$ based on the optimistic estimates $Q^t = \{Q_h^t\}_{h=1}^H$. Under Assumption~\ref{ass:invRKHS_norm}, the estimates $\widehat{Q}_h^t$ and bonus $b_h^t$ are computed via kernel ridge regression.
%%
%
%
% Recall that
% \[
%     \mathbb{E}\left[r_h(z_h^{\tau}) + V^t_{h+1}(s_{h+1}^{\tau})\right] = r_h(z_h^{\tau}) + [P_h V^t_{h+1}](z_h^{\tau}),
% \]
% where the expectation is taken with respect to \( P_h(\cdot | z_h^{\tau}) \). The deviation from this expectation stems from the randomness of transitions, and is bounded by \( H - h \le H \). 
%
%
%
% \subsection{Analysis}
%
% To present our performance guarantees for the KRVI algorithm, we introduce some notation and leverage existing results on kernel complexity.
%
%\subsubsection{Kernel Smoothness}
\begin{algorithm}[H]
\caption{The KOVI Algorithm}
\label{alg:KRVI}
\begin{algorithmic}[1]
\State \textbf{Input:} Regularization \( \lambda \), confidence parameter \( \beta_T(\delta) \), kernel \( k \), MDP \( M = (\Sc, \Ac, H, P, r) \)
\For{episode \( t = 1, 2, \dots, T \)}
    \State Observe initial state \( s_1^t \)
    \State Set \( V^t_{H+1}(s) = 0 \) for all \( s \in \Sc \)
    \For{step \( h = H, H-1, \dots, 1 \)}
        \State Compute \( Q_h^t(z) \) and \( V_h^t(z) \) via kernel ridge 
        \Statex \hspace{2.9em} regression
    \EndFor
    \For{step \( h = 1, 2, \dots, H \)}
        \State Take action \( a_h^t \gets \arg\max_{a \in \Ac} Q_h^t(s_h^t, a) \)
        \State Observe reward \( r_h(s_h^t, a_h^t) \) and next state 
        \Statex \hspace{3em}\( s_{h+1}^t \)
    \EndFor
\EndFor
\end{algorithmic}
\end{algorithm}
\section{THEORETICAL ANALYSIS}
\label{sec:algorithm}

\begin{figure*}[h]
    \centering
    \begin{subfigure}{0.25\textwidth}
        \centering
        \includegraphics[width=\textwidth]{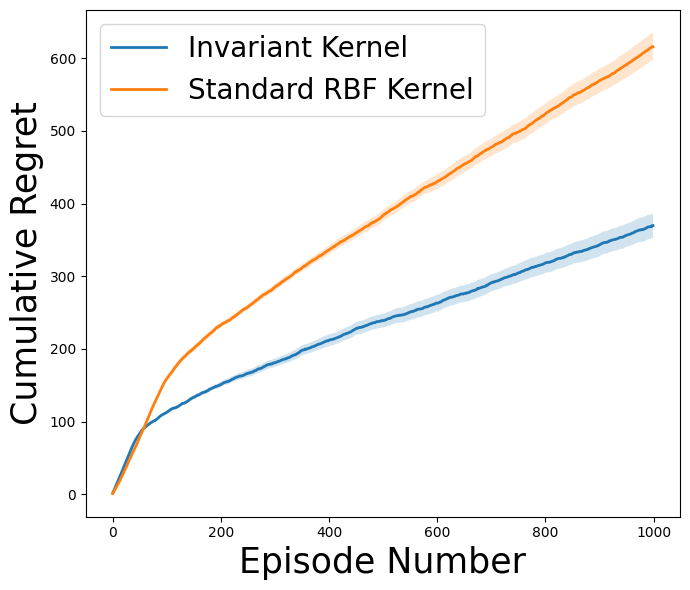}
        %\caption{\parbox[t]{\linewidth}{\centering Synthetic setting\vspace{0.4em}}}

        \caption{Synthetic setting }
        \label{fig:RBF_all_algos}
    \end{subfigure}
    %\hspace{0.3em} 
    \begin{subfigure}{0.21\textwidth}
        \centering
        \includegraphics[width=\textwidth]{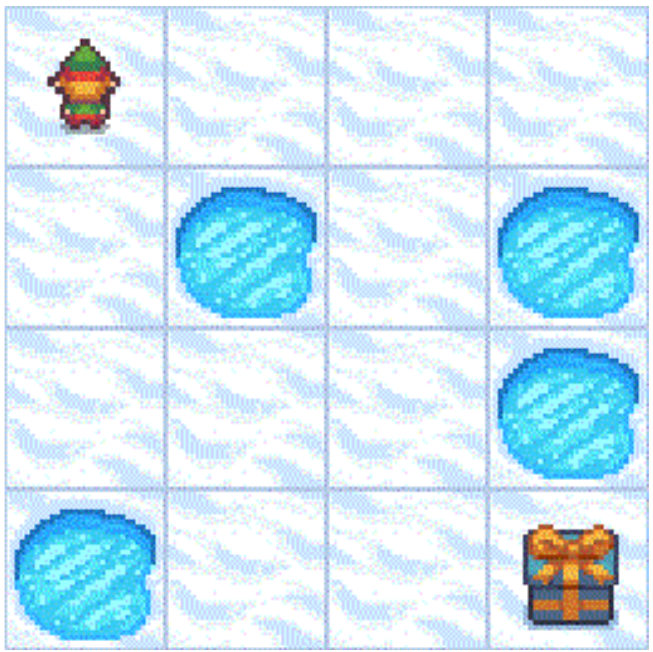} % Replace with the path to your second figure
        \caption{Frozen Lake}
        \label{fig:Frozen Lake environment}
    \end{subfigure}
    \begin{subfigure}{0.25\textwidth}
        \centering
        \includegraphics[width=\textwidth]{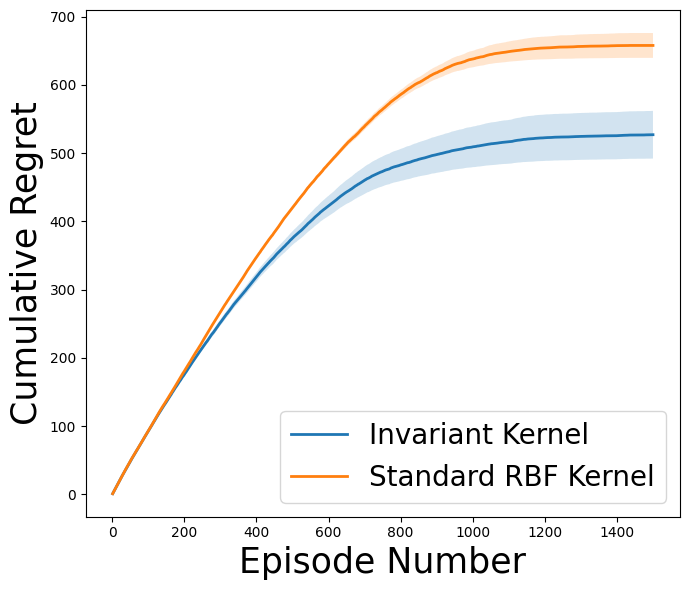} % Replace with the path to your third figure
        \caption{Frozen Lake (Fixed)}
        \label{fig:Matern2.5_all_algos}
    \end{subfigure}
    %\hspace{0.3em} 
    \begin{subfigure}{0.25\textwidth}
        \centering
        \includegraphics[width=\textwidth]{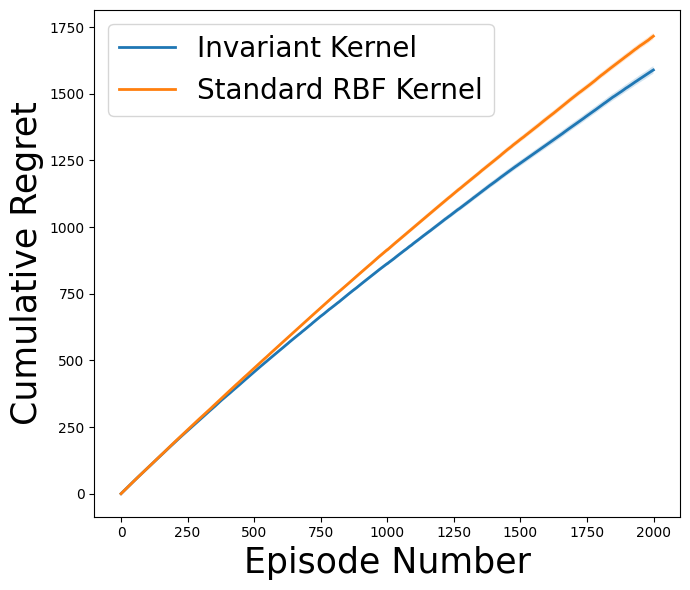} % Replace with the path to your second figure
        \caption{Frozen Lake (Random )}
        \label{fig:Matern1.5_all_algos}
    \end{subfigure}
    \caption{Comparison of KOVI with invariant kernel vs. standard RBF kernel, across different settings. Cumulative regret is plotted against the number of episodes. (a) Regret for synthetic setting.  (b) A rendered frame of Frozen Lake. (c,d) Regret for the Frozen Lake, fixed and random setting, respectively. Regret is averaged over 20 random seeds. Shaded area represents the standard error.}
    \label{fig:overallresults}
\end{figure*}

\begin{table*}[]
\caption{Environments, state and action spaces, and group action used in our experiments.}
\label{tab:environments}
\resizebox{\textwidth}{!}{%
 % Add your title here
\begin{tabular}{llllll}
\cline{1-3}
Environment                  & Space                                                                                                   & Group Action                                                                                                                                                                                                 &  &  &  \\ \cline{1-3}
\multirow{2}{*}{Synthetic}   & S = {[}-1,1{]}                                                                                           & $\{x\to x, x \to -x\}$                                                                                                                                                                                              &  &  &  \\
                             & A = {[}-1,1{]}                                                                                           & $\{x\to x, x \to -x\}$                                                                                                                                                                                              &  &  &  \\ \cline{1-3}  % Proper \cline for alignment
\multirow{3}{*}{Frozen Lake} & \parbox[l]{6cm}{S = {\{}$z=(x_i,y_i)_{i=1}^6 \in \mathbb C^6${\}}}                  & $D_4 = \langle z \to e^{\frac{i \pi}{2}}z, z \to \bar z \rangle \curvearrowright \mathbb C^6$                                                                     &  &  &  \\
                             % &                                                                                                         & $\curvearrowright 90^\circ \circ$ x-axis-reflect, $\curvearrowright 180^\circ \circ$ x-axis-reflect, $\curvearrowright 270^\circ \circ$ x-axis-reflect                                   &  &  &  \\
                             & A = \{ (-1, 0),\ (1, 0),\ (0, -1),\ (0, 1) \}
                                        & \parbox[l]{9cm}{$D_4 \curvearrowright \mathbb C |_A$ } &  &  &  \\ \cline{1-3}
                             
\multirow{3}{*}{Synpl} & \parbox[l]{6cm}{S = {\{}$z=(x_i,y_i)_{i=1}^{16} \in \mathbb C^{16}${\}}}                  & $D_4 = \langle z \to e^{\frac{i \pi}{2}}z, z \to \bar z \rangle \curvearrowright \mathbb C^{16}$                                                                     &  &  &  \\
                             % &                                                                                                         & $\curvearrowright 90^\circ \circ$ x-axis-reflect, $\curvearrowright 180^\circ \circ$ x-axis-reflect, $\curvearrowright 270^\circ \circ$ x-axis-reflect                                   &  &  &  \\
                             & A = {\{}$z=(x,y) \in \mathbb C${\}}                                        & \parbox[l]{9cm}{ {$D_4 \curvearrowright \mathbb C |_A$ } } &  &  &  \\ \cline{1-3}
\end{tabular}%
}
\end{table*}

\subsection{Information Gain} 
We define a kernel-specific 
% problem-independent 
complexity term, referred to as the maximum information gain $\Gamma_k(T)$ from $T$ observations, as follows~\citep{srinivas2009gaussian}:
\begin{equation}\label{eq:maxinfogain}
    \Gamma_k(T) = \max_{\{z_t\}_{t=1}^T \subset \gZ} \log \det\left(I + \frac{1}{\lambda} [k(z_i, z_j)]_{i,j=1}^T \right).
\end{equation}
This expression takes the maximum of the log determinant of a scaled and regularized kernel matrix, with the maximum over all input sequences. It quantifies the complexity of the kernel model and appears in related problems such as BO. In that context, \cite{brown2024sample} provides a bound on the information gain for invariant kernels, formally stated below.
\begin{lemma}[Brown et. al.]\label{lem:infoG}
    Consider the maximum information gain defined in~\eqref{eq:maxinfogain}. Under Assumption~\ref{ass:eigendecay}, we have
    \begin{equation}
        \Gamma_{k_G}(T) = \gO\left( 
        \frac{T^{\frac{1}{p}}}{|G|}
        \right).
    \end{equation}
\end{lemma}
\subsection{Covering Number of Function Class}
Following \cite{yang2020provably}, recall the definition of the set of $Q$-functions in KOVI,
\begin{equation}\label{eq:qfuncclass}
\begin{aligned}
    \gQ(B, R) = \bigl\{ Q^t_h \,\bigm|\,& 
        \forall \{z_j\}_{j=1}^t \subset \gZ, \; t \leq T, \\
        & ||Q_0||_\Hc \leq R, \; \beta \in [0, B] 
    \bigr\}.
\end{aligned}
\end{equation}
%
% where $\hat{Q}_T$ and $\sigma_T$ are the kernel-based prediction and standard deviation conditioned on $\{z_t\}_{t=1}^T$, and $B, R \in\Rr$.

\paragraph{Covering Number:} Consider a set of real functions $\Fc$. For $\epsilon>0$, we define the minimum $\epsilon$-covering set $\Cc(\epsilon)$ as the smallest subset of $\Fc$ that covers it up to an $\epsilon$ error in $l_{\infty}$ norm. That is to say, for all $f\in\Fc$, there exists a $g\in\Cc(\epsilon)$, such that $\|f-g\|_{l_{\infty}}\le \epsilon$. We refer to the size of $\Cc(\epsilon)$ as the $\epsilon$-covering number and use the notation $\gN(\epsilon, B, R, G)$ to denote the $\epsilon$-covering number of $\gQ(B,R)$.  

\begin{theorem}\label{the:cov_num}
    With the notation above, we have:
    \begin{align}
        \log\gN(\epsilon, B, R, G) 
        &= \gO\Biggl( 
            \left(\frac{R^2}{\epsilon^2 |G|^{p}}\right)^{\!\frac{1}{p-1}}
            \left(1+\log\frac{R}{\epsilon}\right) \notag\\
        &\quad + 
            \left(\frac{B^2}{\epsilon^2 |G|^{p}}\right)^{\!\frac{2}{p-1}}
            \left(1+\log\frac{B}{\epsilon}\right)
        \Biggr).
    \end{align}
\end{theorem}

This result provides the explicit dependence of the covering number on the number of symmetries, i.e. all other constants are independent of $\epsilon, B, R$ and $|G|$. See Appendix~\ref{app:proof_cov} for proof.

%\todo[size=\scriptsize,color=red!20!white]{\textbf{Alex}: Please add one paragraph with comments on this theorem. What is the intuition behind it and why is it important?} 

\subsection{Regret Bounds and Discussions}

Recall Theorem~4.2 and the associated notation in~\cite{yang2020provably}, which provides the regret bound in terms of the maximal information gain and covering number.

\begin{theorem}[Yang et. al.]\label{the:regB}
    Consider the episodic MDP described in Section~\ref{sec:episodeMDP}. Under Assumption~\ref{ass:invRKHS_norm}, for the KOVI algorithm, there exist $B$ and $R$ such that for all $T\in\mathbb{N}$ and $\epsilon\in(0,1)$, denoting $\Lambda(k_G; B,R,\epsilon, T) = (\Gamma_{k_G}(T))^{\frac{1}{2}}((\Gamma_{k_G}(T)+\log \gN(\epsilon,B,R,G))T)^{\frac{1}{2}}+T\epsilon$: 
    \begin{equation*}
        R(T) = \gO\left(\Lambda(k_G; B,R,\epsilon, T)\right)
    \end{equation*}
    
\end{theorem}

Under Assumption~\ref{ass:eigendecay}, we may follow the choices in \cite{yang2020function} for $B:=B_T$ and $\epsilon = \epsilon_T$, to extract the dependency of $\Lambda(k_G; B,R,\epsilon, T)$ on $|G|$ explicitly. Namely, Lemma \ref{lem:infoG} Theorem \ref{the:cov_num} give:

\begin{cor}\label{cor:regret}
Under the assumptions above, $||Q^T_h||_\Hc \leq R_T(k_G) := 2H\sqrt{\Gamma_{k_G}}$, for $B_T := H \log(TH)\cdot T^{k^*}$ and $\epsilon = H/T$ we have:
\begin{equation}
    \Lambda(k_G; B_T,R_T(k_G),\epsilon, T) \leq \frac{1}{|G|}\Lambda(k; B_T,R_T(k),\epsilon, T)
\end{equation}
%    and hence:
\begin{equation}
    R(T) = \gO\left(\frac{1}{|G|}\Lambda(k; B_T,R_T,\epsilon, T)\right)
\end{equation}
\end{cor}

%Theorem~\ref{the:regB} is identical to Theorem~4.2 in~\cite{yang2020provably}, which provides the regret bound in terms of the complexity terms \( \Gamma \) and \( \mathcal{N} \). To complete this result, we next derive bounds on these complexity terms in the symmetric RL setting.
 
%\todo[size=\scriptsize,color=red!20!white]{\textbf{Alex}: Please add a paragraph to clarify the relevance of this theorem. If it is identical to Yang et al, why do we include it in our results? What is our contribution here?} 

\paragraph{Remarks.} Substituting $p=1+\frac{2\nu}{d}$ in Corollary \ref{cor:regret}, we have, for large $d$, and any $\nu>0$:
    \begin{equation}
    R(T) = \tilde\gO\left(\frac{1}{|G|} H^2 \cdot T^{\frac{3d(d+1)+2\nu}{4\nu + 2d(d+1)}}\right)
\end{equation}
Contrast this result, with e.g. Theorem \ref{the:hoeffding} or with regret in BO, for which $R_T \sim \sqrt{\Gamma_{k_G}(T)}$ (\cite{wang2024improvedregretboundsbayesian}) which both pose a factor of $\frac{1}{|G|}$ in the bound. Why does KOVI achieve better improvement with $|G|$ over both these bounds? Compared with BO, the suboptimality comes for the acquisition function. Compared to KOVI in the tabular setting, however, the bound is consistent, since both $\mathcal N(\epsilon, B, R)\text{ and } \Gamma(T) \sim SA = |\Sc \times \Ac|$. Choosing $\beta = HSA$ in Theorem \ref{the:regB} gives: $R(T)=\tilde\gO(H^2SA\sqrt{T})$ which is consistent with the $\frac{1}{|G|}$ reduction seen above. 

 \begin{figure*}[h]
    \centering
    \begin{subfigure}{0.30\textwidth}
        \centering
        \includegraphics[width=\textwidth]{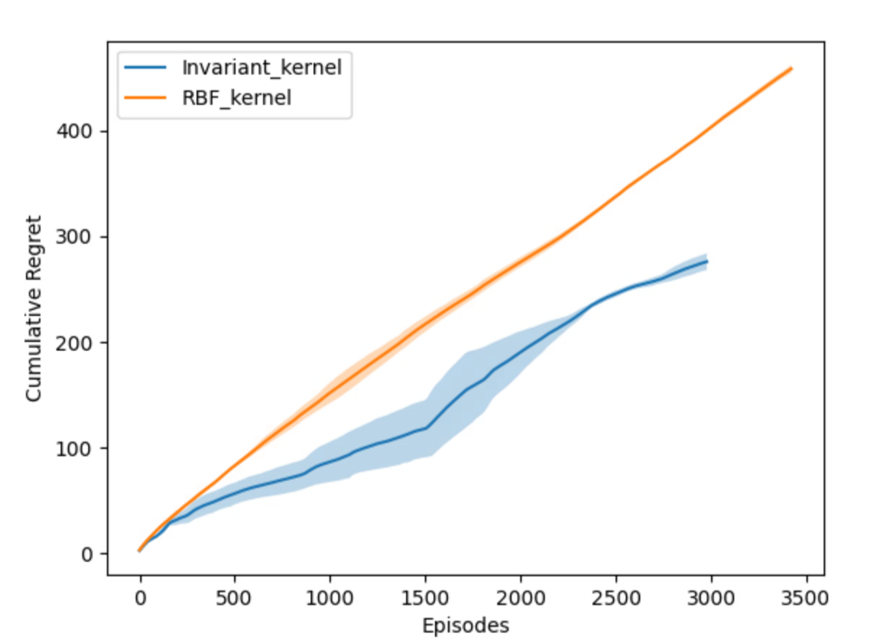}
        \caption{Regret comparison}
        \label{fig:synpl_regret}
    \end{subfigure}
    %\hspace{0.3em} 
    \begin{subfigure}{0.22\textwidth}
        \centering
        \includegraphics[width=\textwidth]{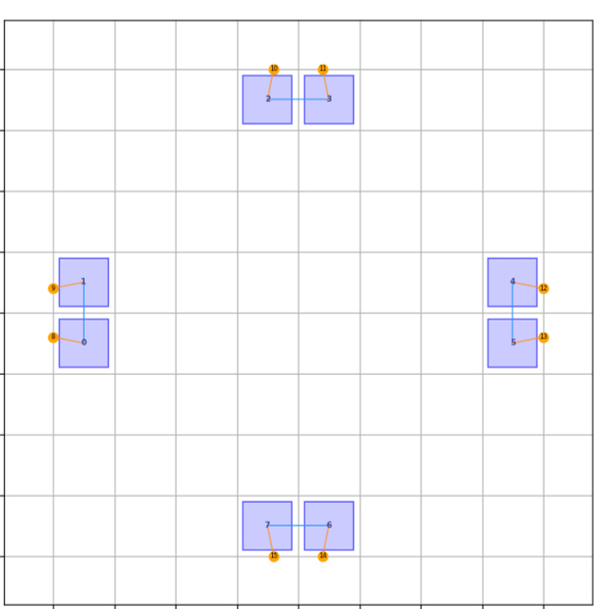}
        \caption{Optimal placement}
        \label{fig:synpl_optimal}
    \end{subfigure}
    \begin{subfigure}{0.22\textwidth}
        \centering
        \includegraphics[width=\textwidth]{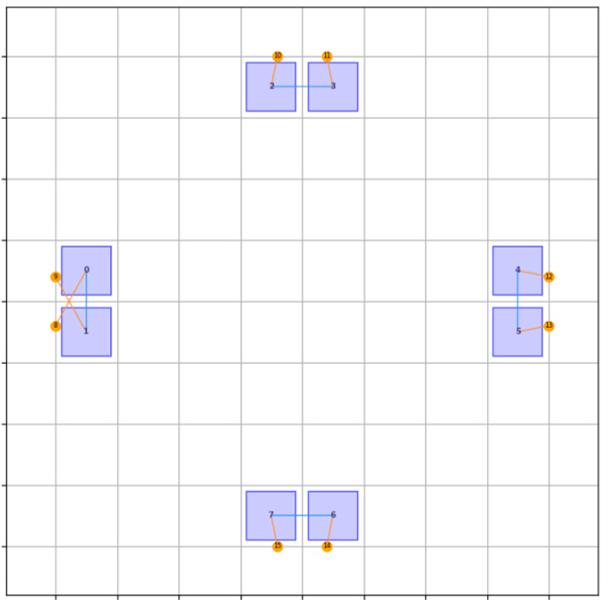}
        \caption{Best placement}
        \label{fig:synpl_best}
    \end{subfigure}
    \begin{subfigure}{0.22\textwidth}
        \centering
        \includegraphics[width=\textwidth]{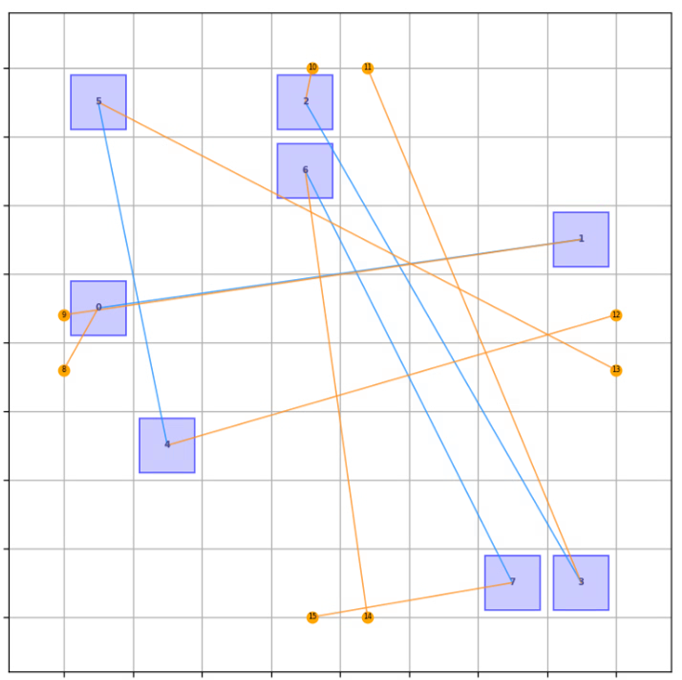} % Replace with the path to your third figure
        \caption{Random placement}
        \label{fig:synpl_random}
    \end{subfigure}
    \caption{Regret comparison of invariant vs. RBF kernel for SynPl (a); Optimal placement (b); Best placement achieved by KOVI (c); Baseline random placement from the random policy (d).}
    \label{fig:synpl}
\end{figure*}

\begin{figure}[h]
    \centering
    \begin{subfigure}[b]{0.4\textwidth}
        \centering
        \includegraphics[width=\textwidth]{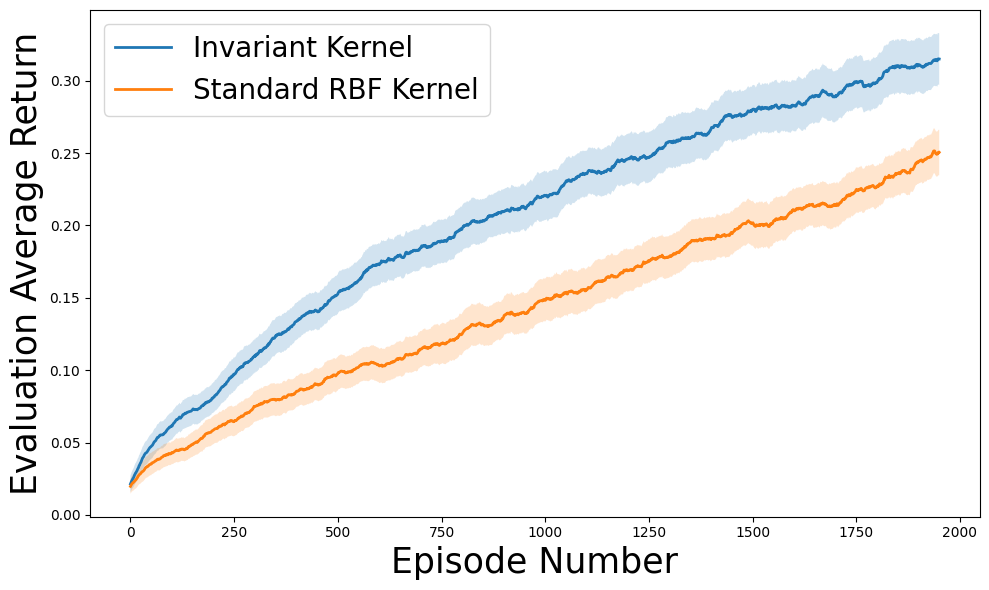}
        \caption{Test performance}
        \label{fig:Test_return_randomized}
    \end{subfigure}
    \hfill
    \begin{subfigure}[b]{0.4\textwidth}
        \centering
        \includegraphics[width=\textwidth]{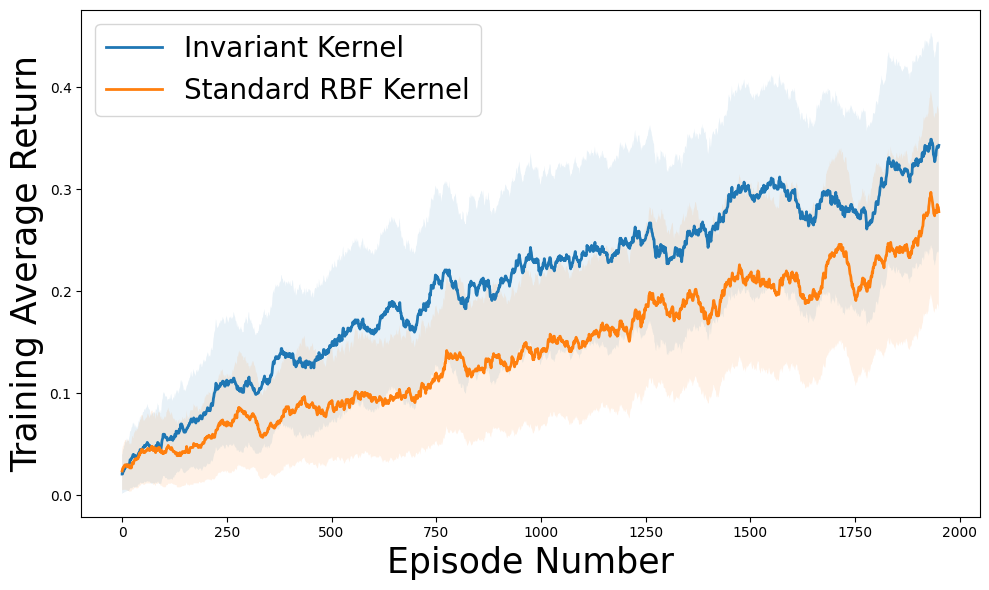}
        \caption{Training Performance}
        \label{fig:Train_return_randomized}
    \end{subfigure}
    \caption{Average return computed on evaluation (test) data (a) and training data (b) vs. number of training episodes for the Random Layout Frozen Lake environment. Shaded areas represent standard error.}
    \label{fig:Test_return}
\end{figure}

\section{EXPERIMENTAL RESULTS}
\label{sec:experiments}
\begin{figure}[h] 
    \centering
    \includegraphics[width=0.4\textwidth]{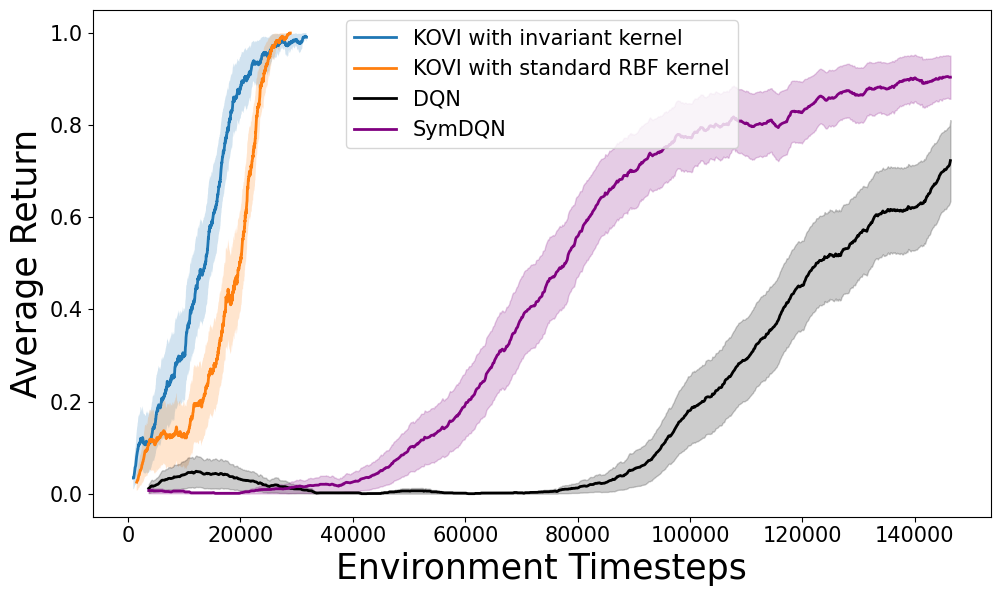}
    \caption{Comparison of KOVI with a standard RBF kernel, KOVI with an invariant kernel, DQN, and SymDQN on the FrozenLake (Fixed) environment. Average episodic return versus number of environment timesteps, averaged over 20 random seeds. The shaded area represents the standard error.}
    \label{fig:DQN_KOVI}
\end{figure}
We demonstrate the performance improvement gained by incorporating known invariances in three settings: (1) a synthetic MDP setting with group symmetry, (2) a symmetrized Frozen Lake environment from the OpenAI Gym library and (3) a 2D placement problem.
The code is available at:
\url{https://github.com/mtkresearch/IQL}.% We implement KOVI using both the invariant kernel and the standard (non-invariant) RBF kernel, and show that the symmetry-aware KOVI algorithm learns more efficiently. We first describe the experimental settings (summarized in Table~\ref{tab:environments}) and then present regret comparison plots (Figures~\ref{fig:overallresults} and~\ref{fig:synpl}).
\paragraph{Synthetic Setting:} We choose $H=10$ and $\mathcal{S}=\mathcal{A}=[-1,1]$ discretized into 10 evenly spaced points. The group we consider has size $|G| = 2$ and includes the identity and inversion transformations. The reward function $r$ and transition function $P$ are constructed to be invariant under $G$, by sampling functions from the RKHS of the invariant kernel $k_G$ (eq.~\ref{invariant_kernel}). We choose RBF as the base kernel $k$, and we train for a total of 1000 episodes. For a more detailed explanation of how $r$ and $P$ are generated and the hyperparameters used, please refer to Appendix~\ref{app:synthetic}. 
\paragraph{Frozen Lake:} We build a navigation task based on OpenAI Gym’s FrozenLake-v1 environment which is invariant to the 8-fold symmetry group of the square, $D_4$. The environment itself depends on a layout of holes, start and goal positions, see Figure~\ref{fig:Frozen Lake environment}. 
Each state is represented as a concatenation of 2-dimensional coordinate vectors 
$(x,y)$ specifying the positions of the agent, goal, and four holes, resulting in a 
6×2 state vector. Actions are represented as 2D discrete coordinate vectors.
%($\leftarrow = [-1,0]$, $\downarrow= [0,-1]$, $\rightarrow = [1,0]$, $\uparrow = [0,1]$). 

At each episode, the agent is tasked to solve the problem in any one of the 8 possible transformations of the layout according to $D_4$, sampled uniformly at random. In Figure \ref{fig:overallresults}, this is denoted by \textit{Fixed Layout}, denoting the fact that the only change in environment that the agent observes is due to the presence of symmetry in $D_4$. 
%Moreover, for the agent to have a fair chance at solving the environment, we must pass a state-action representation that is sensitive to the action of $D_4$, without trivializing it (like the discrete representation would). Hence, 
% To introduce symmetry, we apply eight group transformations—composed of four-fold rotations and reflections—simultaneously to both the state and action spaces, resulting in a group transformations of size $|G|=8$. These transformations are applied to the default 4×4 map to produce eight different grid layouts. We refer to this customized environment as ``Rotate-Reflect Frozen Lake'', and at the start of each episode, one of these layouts is sampled uniformly at random. 
In order to exploit a more realistic scenario, where symmetry naturally occurs in the environment, we also train KOVI on a variant of Frozen Lake with fully randomized Layouts. In Figure \ref{fig:overallresults} we refer to these experiments as \textit{Random Layout}.
In this setting, a new layout of the starting point and holes is randomly generated at the beginning of each episode.
Four holes are placed randomly, and only configurations that allow a valid path from the start to the goal state are accepted. 
%The starting state remains in one of the four corners of the layout, while the goal state is always in the opposite corner. 
This setting is significantly more challenging due to the greater diversity of possible configurations. We train for a total of 2000 episodes. For details about the hyperparameters used, please refer to Appendix~\ref{app:frozen_lake}. 

In Figure ~\ref{fig:overallresults} we plot the cumulative regret during training on the synthetic setting, the Fixed and Random Layout Frozen Lake. For the latter only, we also evaluate the episodic return on 40 randomly sampled test environments and report the average performance on these test environments in each training episode. %The corresponding plot is also included in Figure~\ref{fig:Test_return} of Appendix~\ref{app:frozen_lake}. 
Figure~\ref{fig:Test_return_randomized} shows the average return on these test environments over the course of training, averaged across 20 random seeds. While the return has not fully converged— due to the complexity introduced by the large number of random grid configurations—the invariant kernel consistently outperforms the standard RBF kernel. For reference, we also plot the average return on the training environment in Figure~\ref{fig:Train_return_randomized}.
\paragraph{Comparison vs Neural Networks}
In Figure \ref{fig:DQN_KOVI} we present a comparison of relative sample complexity of our algorithm versus an Q-Network on FrozenLake implemented with equivariant MLP layers. The kernel regression enjoys consistent improvements in sample complexity versus the more data inefficient Q-network. See additional experimental details in Appendix \ref{appendix:dqn_v_kovi}.
\paragraph{2-Dimensional Placement:} The placement problem can be described as a 2D planar embedding of a graph of components under geometric constraints. This problem appears frequently in contexts such as smart-city design, building floorplanning, electronic design and VLSI. We build a simplified environment where we can study this task, denoted by \textit{SynPl} in Figure \ref{fig:synpl}.
The task is to correctly place 8 units on an 8x8 grid with the following constraints: 1. No two units should overlap. 2. Each edge of the graph should be enclosed in a minimal area subset of the grid. 
While the placement task may be easy to solve for humans by inspection, leveraging our vast compendium of geometric reasoning and intuition, even simple cases remain difficult for machines without carefully tailored backend representations (see, e.g. GNNs used to encode the graphical information \cite{mirhoseini2021graph}, \cite{chang2022flexible}). 1-shot placement methods like BO perform poorly in this setting due to:
\begin{enumerate}
    \item The subset of invalid configurations lies at the boundary of the space of optimal configurations (where the edge length is minimized). If overlaps are handled via penalties in the optimization objective, this requires modeling with progressively less smooth kernels, which rapidly decreases sample complexity guarantees. Forcing the acquisition function to only sample valid configurations (according to the posterior) is not without its own issues, with rejection sampling being the most straightforward, yet cumbersome approach.
    \item The configuration space naturally lies on a $2n$ dimensional manifold in $\mathbb R^{2n}$. While lower-dimensional embeddings of this data exist, they don't benefit from the natural action of the isometry group on $\mathbb R^{2n}$, missing the crucial geometric information in the configuration space. Moreover, such embeddings represent a strenuous design decision on the experimenter. On the other hand, as $n$ increases, this leads to high-dimensionality BO, which carries its own sample complexity issues (\cite{eriksson2021high}).
\end{enumerate}
A better design decision is to leverage the sequential nature of the placement problem to produce valid configurations with progressively higher objective values $\Phi(s)$.
The most general form of the environment would sample all canvas configurations uniformly
%and thus benefit from natural symmetry w.r.t. $D_4$, 
however we consider a simplified setting 
%meant to highlight the difference in the two kernels, 
and 
fix a single problem with a unique symmetric solution. 
Specifically, the environment places the units sequentially, with each action being the selection of one of the 64 grid cells. An action mask is applied to avoid overlaps. 
The state, at any one point consists of the position of the preplaced units in Cartesian coordinated, with padding for the yet-to-be-placed units.
The reward at each step $t$ is a function $r(s_t,a_t)$, defined as follows. Let $\pi^*$ be an oracle policy for all states, which is able to optimally place all remaining units in all states. Then $r(s_{t-1},a_{t-1}) = \Phi(\pi^*(s_t)^{[H-t]}) - \Phi(\pi^*(s_{t-1})^{[H-t+1]})$, the potential difference due to selecting $a_{t-1}$. As this quantity is inaccessible, we replace the oracle policy $\pi^*$ by a known reference policy $\pi_0$. As this need not be a deterministic policy, we sample several trajectories from $\pi_0$ and average the $\Phi$ estimates correspondingly. As the horizon for our problem is low ($H=8$), in our experiments we take $\pi_0$ to be the random policy, and do not observe significant degradations due to this choice.
In our experiments we observe that the invariant kernel consistently produces lower regret, even in scenarios where symmetries are not inherently built into the environment. Empirical results differ from their theoretical counterparts in that the asymptotic expected gap isn't always recovered for small $T$. Our experiments thus strengthen the theoretical results by providing insight into the behavior of KOVI with and without invariant kernel in finite time. See Figure \ref{fig:synpl} for regret and placements.
\subsection{Limitations}
\label{sec:lim}
The KOVI algorithm suffers from general limitations which depend either on its theoretical formulation or on implementation. The two principal points are the choice of hyperparameters, primarily $\beta$ and the overall complexity of the algorithm. 
Firstly, notice that since at every step $H$ GP models are fitted to an experience buffer which contains all samples acquired to that point, the complexity of each step is $\gO(Ht^3)$ due to inverting the kernel matrix, and summing over $t \in [0,T]$ produces a time complexity of $\gO(HT^4)$. This is prohibitive for realistic environments for which more than a couple thousand samples need to be collected.
This however can be remedied as follows: 1. The time complexity can be brought down to $\gO(Ht^3)$, matching that of BO with exact kernel inversion, by caching previous calculations in the Cholesky decomposition of each GP model and performing only rank 1 updates. 2. Secondly, sparse GP methods can be applied (with or without optimizing the inducing point acquisition) to compress the collected samples into a manageable size for kernel inversion. 
In terms of algorithmic stability, the parameter $\beta$ is empirically identified as the main contributor.
%, the kernel parameters such as the RBF kernel's length scale and noise regularization can be tuned via BO at each iteration of the algorithm. In practice they are not the main source of instability, although further efforts need to be taken to leverage previous value of the kernel parameters if rank 1 updates were applied to the Cholesky decomposition.
In practice it is kept fixed throughout our runs and optimized as a hyperparameter, which is in contrast to the choice of $\beta = B_T$  in the theoretical formulation (which, in the tabular setting scales as $HSA$). This likely produces suboptimal results, however, more research needs to be done for adaptively setting $\beta$ in terms of $T$ (or $|G|$) in this setting.
In terms of convergence properties of the algorithm, the $\beta \cdot \sigma_T(z)$ exploration bonus, induces a mode-hopping behavior, by over-estimating $Q$-values early in the runtime. Sometimes optimal solutions are found and all neighboring solutions are explored early on in the training process, however, the algorithm subsequently selects different, less explored configurations in an effort to reduce uncertainty there. This is not problematic in synthetic settings where the optimal $V^*(s)$ is known ($s$ still hidden), however, in realistic settings, this behavior can contribute to a large proportion of the runtime.
\section{CONCLUSION}
\label{sec:conc}
Group symmetry in MDPs provides powerful structural constraints that, when incorporated effectively into the RL pipeline—via equivariant policies and invariant function spaces—can significantly reduce sample complexity. 
This work is the first to quantify the sample efficiency gains of incorporating group invariances into a fully invariant kernel within the kernel-based RL setting. We derive novel bounds on the maximum information gain and covering number for the invariant kernel, and demonstrate that our improved regret bounds are supported empirically through experiments on synthetic domains, classical RL benchmarks, and a practical 2D placement task, highlighting real-world applicability. For future work, it would be valuable to evaluate our symmetry-aware LSVI algorithm on more complex or high-dimensional input domains, as well as to extend our theoretical results to settings with partial or approximate symmetries.

\clearpage
\bibliography{references}
%%%%%%%%%%%%%%%%%%%%%%%%%%%%%%%%%%%%%%%%%%%%%%%%%%%%%%%%%%%%
\section*{Checklist}

% %%% BEGIN INSTRUCTIONS %%%

\begin{enumerate}

  \item For all models and algorithms presented, check if you include:
  \begin{enumerate}
    \item A clear description of the mathematical setting, assumptions, algorithm, and/or model. [Yes, the mathematical setting is introduced in Section~\ref{sec:prelim}, which covers Episodic Markov Decision Processes (Section~\ref{sec:episodeMDP}), Symmetric RL (Section~\ref{sec:symmetric_RL}), and Kernel-Based RL with Symmetries (Section~\ref{sec:kernel_RL_with_sym}). In the same section, we also state our technical assumptions, namely Assumptions~\ref{ass:invRKHS_norm} and~\ref{ass:eigendecay}. The LSVI algorithm is then described in Section~\ref{sec:lsvi}.  ]
    \item An analysis of the properties and complexity (time, space, sample size) of any algorithm. [Yes, a theoretical analysis of LSVI with a totally invariant kernel, leading to novel bounds on sample complexity, is provided in Section~\ref{sec:algorithm}. In addition, we discuss the time complexity of our approach in the limitations section (Section~\ref{sec:lim}).]
    \item (Optional) Anonymized source code, with specification of all dependencies, including external libraries. [Yes, the code used for the experiments in the synthetic setting and on the Frozen Lake environment is publicly available at:
\url{https://github.com/mtkresearch/IQL}. For the experiments on the 2D placement problem, we provide a clear and detailed description of the setup in Section~\ref{sec:experiments}, but we do not share the corresponding code, as it is confidential.]
  \end{enumerate}

  \item For any theoretical claim, check if you include:
  \begin{enumerate}
    \item Statements of the full set of assumptions of all theoretical results. [Yes, we present in Section~\ref{sec:prelim} our technical assumptions, namely Assumptions~\ref{ass:invRKHS_norm} and~\ref{ass:eigendecay}. ]
    \item Complete proofs of all theoretical results. [Yes, formal proofs for all theorems and propositions stated in the main text are provided in the appendix. The proof of Proposition~\ref{prop:equivinv} can be found in Appendix~\ref{appendix:symrl}, while Proposition~\ref{prop:quotient_MDP} is stated and proved in Appendix~\ref{prop:quotient_MDP}. Theorem~\ref{the:hoeffding} is stated and proved in Appendix~\ref{appx:tabular_Q}, and Theorem~\ref{the:cov_num} is proved in Appendix~\ref{app:proof_cov}. Where our work builds upon lemmas or theorems from prior literature, the original sources are clearly cited, and their relevance to our setting is explained (e.g., Theorem~\ref{the:regB}).]
    \item Clear explanations of any assumptions. [Yes, our main assumptions are clearly explained in Section~\ref{sec:kernel_model_RL}.]     
  \end{enumerate}

  \item For all figures and tables that present empirical results, check if you include:
  \begin{enumerate}
    \item The code, data, and instructions needed to reproduce the main experimental results (either in the supplemental material or as a URL). [Yes, the code used for the experiments in the synthetic setting and on the Frozen Lake environment is publicly available at \url{https://github.com/mtkresearch/IQL}. However, since the
2D placement experiment relies on proprietary code and models, the corresponding
implementation cannot be disclosed or made
publicly available.] 
    \item All the training details (e.g., data splits, hyperparameters, how they were chosen). [Yes, for comprehensive details—such as hyperparameter tuning, visualizations of the synthetic reward function, and computational resources used—please refer to Appendix~\ref{app:exp}. ]
    \item A clear definition of the specific measure or statistics and error bars (e.g., with respect to the random seed after running experiments multiple times). [Yes, in all our experiments, we report relevant statistical information where applicable. This includes the sample sizes for experimental runs and error bars representing the standard error.]
    \item A description of the computing infrastructure used. (e.g., type of GPUs, internal cluster, or cloud provider). [Yes,  we make special mention of the type of compute and memory we use in Appendix~\ref{app:exp}.]
  \end{enumerate}

  \item If you are using existing assets (e.g., code, data, models) or curating/releasing new assets, check if you include:
  \begin{enumerate}
    \item Citations of the creator If your work uses existing assets. [Yes, we have cited the Scikit-Learn and BoTorch libraries used in our experiments (see Appendix~\ref{app:exp}). ]
    \item The license information of the assets, if applicable. [Not Applicable]
    \item New assets either in the supplemental material or as a URL, if applicable. [Yes, the code generating our experimental results is publicly available at \url{https://github.com/mtkresearch/IQL}.]
    \item Information about consent from data providers/curators. [Not Applicable]
    \item Discussion of sensible content if applicable, e.g., personally identifiable information or offensive content. [Not Applicable]
  \end{enumerate}

  \item If you used crowdsourcing or conducted research with human subjects, check if you include:
  \begin{enumerate}
    \item The full text of instructions given to participants and screenshots. [Not Applicable]
    \item Descriptions of potential participant risks, with links to Institutional Review Board (IRB) approvals if applicable. [Not Applicable]
    \item The estimated hourly wage paid to participants and the total amount spent on participant compensation. [Not Applicable]
  \end{enumerate}

\end{enumerate}

\clearpage
\appendix
\thispagestyle{empty}

% Supplementary material: To improve readability, you must use a single-column format for the supplementary material.
\onecolumn

\section{Group Action}
\label{appx:grp_action}

% In many real-world RL problems, the environment exhibits symmetries: structured transformations of states and actions that leave the reward and dynamics unchanged. In this section, we formalize this notion and provide intuitions for how symmetry can be leveraged to reduce sample complexity. We defer detailed derivations and extended discussions to Appendix~\ref{app:symmetric-mdp}.

 We recall the standard definition. Let $(G, \cdot)$ be a group acting on a space $\X$ via a binary operation $\circ: G \times \X \to \X$ such that:
\begin{itemize}
    \item $g \circ (h \circ x) = (g \cdot h) \circ x$ for all $g,h \in G$, $x \in \X$
    \item $e \circ x = x$ for identity element $e \in G$
\end{itemize}
%\section{Proofs from Section~\ref{sec:symmetricRL}}
\section{Proof of Proposition~\ref{prop:equivinv}}
\label{appendix:symrl}
\begin{proof}%[Proof of Proposition~\ref{prop:equivinv}]
    Let $H$ be the horizon of the MDP. We will superscript states by $t = 0,... H$ to indicate the time-step it is observed in a particular trajectory.
    We'll show this by induction, starting backwards from a terminal state $s^H$, for which, by definition $Q_\pi(s^H,a) = V_\pi(s^H)=0$.
    Now for an arbitrary state $s^t$, we have:
    \begin{align}
        Q_\pi(gs^t,ga) &= \sum_{s^{t+1}\in\mathcal S} \mathbb P(s^{t+1}|gs^t, ga)[r(gs^t,ga,s^{t+1}) + V_\pi(s^{t+1})]\\
        &= \sum_{s^{t+1}\in\mathcal S} \mathbb P(g^{-1}s^{t+1}|s^t, a)[r(s^t,a,g^{-1}s^{t+1}) + V_\pi(s^{t+1})]\\
        &= \sum_{s^{t+1}\in\mathcal S} \mathbb P(s^{t+1}|s^t, a)[r(s^t,a,s^{t+1}) + V_\pi(gs^{t+1})]
    \end{align}
    The value function itself satisfies:
    \begin{align}
        V_\pi(gs^{t+1}) &= \sum_{a \in \mathcal A}\pi(a|gs^{t+1}) Q_\pi(gs^{t+1},a) = \\
        &= \sum_{a \in \mathcal A}\pi(g^{-1}a|s^{t+1}) Q_\pi(gs^{t+1},a)\\
        &=\sum_{a \in \mathcal A}\pi(a|s^{t+1}) Q_\pi(gs^{t+1},ga)\\
        &= \sum_{a \in \mathcal A}\pi(a|s^{t+1}) Q_\pi(s^{t+1},a) = V_\pi(s^{t+1}),
    \end{align}
where the final two equalities follow from the induction hypothesis. 

The final element is showing greedy policies of invariant $Q$-functions are equivariant.
\begin{align}
    \pi_q(gs) = \argmax_{a} q(gs,a) = g^{-1}\argmax_{a} q(gs, ga) = g^{-1}\argmax_{a} q(s, a) = g^{-1}\pi_q(s)
\end{align}

\end{proof}
\section{Proposition~\ref{prop:quotient_MDP}}
\label{appx:prop:quotient_MDP}

Denote the orbit of an element $x$ under the action of $G$ by the set $Orb(x) = \{g(x)| g\in G\}$. Under a couple of simplifying assumptions, the dynamics of the MDP reduces to that of a simpler MDP with strictly smaller state space.

\begin{prop}\label{prop:quotient_MDP}
Assume $\mathcal M$ is a $G$-invariant MDP. If the action on $\mathcal A$ is trivial, i.e. 
\begin{align}\label{eq:trivial-action-action}
    g \circ_\Ac a = a, \forall g \in G,
\end{align} then the following quantity
\begin{equation}\label{eq:orb-dynamics}
    \mathbb P(Orb(s')|Orb(s),a):= \mathbb P(Orb(s')|s,a)
\end{equation}
is well-defined and $\mathcal M/G = (\mathcal S/G, \Ac, H, \mathbb P, r_G)$ given by
$\mathcal S/G = \{Orb(s)|s\in \mathcal S\}$, dynamics $\mathbb P$ given by \ref{eq:orb-dynamics}, and reward given by $r_G(Orb(s),a) = r(s,a)$ is a well-defined MDP and the optimal policy $\pi^*_{\mathcal M/G}$ extends to an optimal policy for $\mathcal M$, by 
\begin{align}\label{eq:invariant-extension}
    \pi^*_{\mathcal M}(s) = \pi^*_{\mathcal M/G}(Orb(s)).
\end{align}

\end{prop}

\begin{proof}
    Equation \ref{eq:orb-dynamics} effectively tells us that the probability of transitioning from one state to another, can be aggregated across orbits in such a way that the corresponding actions are unchanged.
    Namely given $s_1 = g(s_2)$:
    \begin{align*}
        \mathbb P(Orb(s')|Orb(s_1),a) = \mathbb P(Orb(s')|s_1,a) &= \mathbb P(Orb(s')|g(s_2),a) = \\ = \mathbb P(g^{-1}(Orb(s'))|s_2,a) &= \mathbb P(Orb(s')|s_2,a) = \mathbb P(Orb(s')|Orb(s_2),a)
    \end{align*}
    Note that we have used Equation \ref{eq:trivial-action-action} here. Now let $\pi^*_{\mathcal M/G}$ be an optimal policy for ${\mathcal M/G}$, and $V^*_{\mathcal M}$ be the optimal value function for $\mathcal M$. Due to Proposition \ref{prop:equivinv}, $V^*_{\mathcal M}$ is invariant, so $V:=V(Orb(s)):= V^*_{\mathcal M}(s)$ is well defined. Moreover, by construction, $V$ is a value function corresponding to the policy $\pi^*_{\mathcal M}(Orb(s))$. Now by optimality of $\pi^*_{\mathcal M/G}$ for all $s$ we have $V(Orb(s)) \leq V^{\pi^*_{\mathcal M/G}}(Orb(s))$. Conversely, $V^{\pi^*_{\mathcal M/G}}$ defines a valid value function on $\mathcal M$ by $V^{\pi^*_{\mathcal M/G}}(s) =V^{\pi^*_{\mathcal M/G}}(Orb(s))$, and hence, by optimality, $V^{\pi^*_{\mathcal M/G}}(s)\leq V^*_{\mathcal M}(s)$ for all $s$. This shows that the two value functions are equivalent, that $\pi^*_{\mathcal M/G}(s):= \pi^*_{\mathcal M/G}(Orb(s))$ is optimal for $\mathcal M$.
\end{proof}

\section{Tabular Q-learning with symmetry}
\label{appx:tabular_Q}
In recent years, there has been a family of research outputs which study the complexity of Q-learning algorithms in the tabular setting. Typical results such as \cite{jin2018q}, give complexity bounds in the tabular setting in terms of the size of the state-action space, $|\Sc \times \Ac| = SA$. 
In our setting, even though we cannot define a simpler dynamics on $\Sc/G \times \Ac/G$, we will see that the regret bounds given in terms of $SA$ extend to regret bounds in terms of $\tilde{SA} = |(\Sc \times \Ac)/G|$.

The gains in sample complexity in general will come from how new state action pairs are observed from unrolling equivariant policies. 
We will denote a trajectory under $\pi$ by $\mathcal T$ to be a sequence $(s_i,a_i)|_{i=1}^H \in \mathcal Z$ with the property that $a_i \sim \pi(s_i)$ and $s_{i+1} \sim \mathbb P(s_i, a_i)$, The corresponding action of $G$ on the space of $\mathcal T$'s is given by $g(\mathcal T) = \left(g(s_i),g(a_i)\right)|_{i=1}^H$. We then have that if $\pi$ is equivariant, the following holds:
\begin{enumerate}\label{enum:inv-traj}
    \item The returns of $\mathcal T$ and $g\mathcal T$ are the same.
    \item $\mathbb P_\pi(\mathcal T) = \mathbb P_\pi(g\mathcal T) $.
\end{enumerate}

We can use these properties to extend algorithms such as Algorithm \ref{alg:hoeffding-sym}.

\begin{algorithm}
\caption{Q-learning with UCB-Hoeffding \textcolor{black}{\& symmetric experience augmentation.}}
\label{alg:hoeffding-sym}
\begin{algorithmic}[1]
\State Initialize $Q_h(s, a) \leftarrow H$ and $N_h(s, a) \leftarrow 0$ for all $(s, a, h) \in S \times A \times [H]$.
\For{episode $k = 1, \ldots, K$}
    \State Receive $s_1$.
    \For{step $h = 1, \ldots, H$}
        \State Take action $a_h \leftarrow \arg\max_{a'} Q_h(s_h, a')$, and observe $s_{h+1}$.
        \textcolor{black}{\For{$(s,a) \in G(s_h, a_h)$}
            \State $t = N_h(s, a) \leftarrow N_h(s_h, a_h) + 1$; $b_t \leftarrow c \sqrt{\frac{H^3 \iota}{t}}$.
            \State $Q_h(s, a) \leftarrow (1 - \alpha_t)Q_h(s_h, a_h) + \alpha_t \left[r_h(s_h, a_h) + V_{h+1}(s_{h+1}) + b_t \right]$.
            \State $V_h(s) \leftarrow \min\{H, \max_{a' \in A} Q_h(s_h, a')\}$.
        \EndFor}
    \EndFor
\EndFor
\end{algorithmic}
\end{algorithm}

The regret bounds directly translate to:

\begin{theorem}\label{the:hoeffding}
There exists an absolute constant $c > 0$ such that, for any $p \in (0, 1)$, if we choose $b_t = c \sqrt{\frac{H^3 \iota}{t}}$, then with probability $1 - p$, the total regret of Q-learning with UCB-Hoeffding and symmetric experience augmentation (see Algorithm \ref{alg:hoeffding-sym}) is at most $O(\sqrt{H^4 \tilde{SA}T \iota})$, where $\iota := \log\left(\frac{SAT}{p}\right)$.
\end{theorem}

\begin{proof}[Proof of Theorem~\ref{the:hoeffding}]

\emph{Proof:} The update to the Q-function is now:
\begin{equation}
Q_{h}^{k+1}(s, a) = 
\begin{cases} 
(1 - \alpha_t)Q_{h}^{k}(s, a) + \alpha_t \left[r_h(s, a) + V_{h+1}^{k}(s_{h+1}^{k}) + b_t \right] & \text{if } (s, a) \in G(s_{h}^{k}, a_{h}^{k}) \\
Q_{h}^{k}(s, a) & \text{otherwise}
\end{cases}
\end{equation}
so the Q-function remains constant on orbits of $G$.
Because of this formulation, $Q_{h}^{k}(s, a)$ is updated once per episode - horizon pair, the same as in \cite{jin2018q}. This means that all results and bounds that apply to a single $(s,a)$ pair follow through. In particular this applies to Lemmas 4.2 and 4.3 from \cite{jin2018q}.

Now, turning to their proof for the result, the quantities:
\begin{equation}
\delta_{h}^{k} := (V_{h}^{k} - V_{h}^{\pi^k})(gs_{h}^{k})
\end{equation}

\begin{equation}
\phi_{h}^{k} := (V_{h}^{k} - V_{h}^{*})(gs_{h}^{k})\text{ and}
\end{equation}

\begin{equation}
t = n_h^k = N_{h}^{k}(gs_{h}^{k}, ga_{h}^{k})
\end{equation}
are still well defined and independent of $g\in G$.

We need to bound $\sum_{k=1}^K \delta_h^k$.

From:
\begin{equation}
\delta_h^k \leq \alpha^{0}_t H + \sum_{i=1}^{t} \alpha^{i}_t \phi_{h+1}^{k_i} + \beta_t - \phi_{h+1}^{k} + \delta_{h+1}^{k} + \xi_{h+1}^{k}
\end{equation}

We bound the summation of the first term:
\begin{equation}
\sum_{k=1}^{K} \alpha^{0}_{n_{h}^{k}} H = H \cdot \sum_{k=1}^{K}  \mathbb{I}[n_{h}^{k} = 0] \leq \tilde{SA}H
\end{equation}
where the final inequality is due to the fact that  $Q_h(s, a)$ gets updated on entire orbits of $G$ on $\Sc \times \Ac$, so either $N_{h}^{k}(s,a) > 0$ for all $(s,a) \in G(s_h, a_h)$, or $N_{h}^{k}(s,a) = 0$ for all $(s,a) \in G(s_h, a_h)$ in which case $Q_h(s, a) = H$ and so $\arg\max_{a'} Q_h(s_h, a') = H$. By \ref{enum:inv-traj}, in the first case, we know that $N_{h-1}^{k}(s,a) > 0$ for all $(s,a) \in G(s_{h-1}, a_{h-1})$. Unrolling the trajectory backwards, we get $N_{1}^{k}(s_1,a) \geq 1$, which can only happen if all $a\in \Ac$ had been taken in $s_1$ prior to episode $k$. 

Now, again, following the argument in \cite{jin2018q}, we get:
\begin{equation}
\sum_{k=1}^{K} \delta_{1}^{k} \leq O \left( H^2 SA + \sum_{h=1}^{H} \sum_{k=1}^{K} (\beta_{n_{h}^{k}} + \xi_{h+1}^{k}) \right)
\end{equation}

Subsequently:

\begin{equation}
\sum_{k=1}^{K} \beta_{n_{h}^{k}} \leq O(1) \cdot \sum_{k=1}^{K} \sqrt{\frac{H^3 \iota}{n_{h}^{k}}} = O(1) \cdot \sum_{(\bar s,\bar a) \in (\Sc \times \Ac) / G} \sum_{n=1}^{N_{h}^{K}(s,a)} \sqrt{\frac{H^3 \iota}{n}} \leq O\left(\sqrt{H^3 \tilde{SA} K \iota}\right) = O\left(\sqrt{H^2 \tilde{SA}T \iota}\right)
\end{equation}

and the final inequality follows due to the fact that $\sum_{(\bar s,\bar a) \in (\Sc \times \Ac) / G} N_{h}^{K}(s, a) = K$ and the left hand side is maximized for $N_{h}^{K}(s, a) = K / \tilde{SA}$.
\end{proof}

\section{Kernel Methods}\label{appendix:kernel}
\paragraph{Mercer Theorem}
Let $k: \Zc \times \Zc \rightarrow \R$ be a positive definite kernel. Its associated reproducing kernel Hilbert space (RKHS), $\Hc_k$, consists of functions $f: \Zc \to \R$ with inner product $\langle \cdot, \cdot \rangle_{\Hc_k}$ and norm $|\cdot|_{\Hc_k}$, satisfying the reproducing property $f(z) = \langle f, k(\cdot, z) \rangle_{\Hc_k}$.
By Mercer’s theorem, $k$ admits an eigenfunction expansion
\begin{equation}
k(z, z') = \sum_{m=1}^{\infty} \lambda_m \phi_m(z) \phi_m(z'),
\end{equation}
where $\{\lambda_m\}_{m=1}^\infty$ are positive eigenvalues and $\{\phi_m\}_{m=1}^\infty$ are orthonormal eigenfunctions in $\Hc_k$. Any $f \in \Hc_k$ can be expressed as $f = \sum_{m=1}^\infty w_m \sqrt{\lambda_m} \phi_m$, with norm $\|f\|_{\Hc_k}^2 = \sum_{m=1}^\infty w_m^2$.

\paragraph{Kernel-Based Prediction.} For $f \in \Hc_k$ and data ${(z_i, y_i)}_{i=1}^t$ with $y_i = f(z_i) + \varepsilon_i$ and zero-mean noise, kernel ridge regression yields posterior mean and variance:
\begin{align}
\hat{f}_t(z) = k_t^\top(z) (K_t + \lambda I)^{-1} \bm{y}_t, \quad
\sigma_t^2(z) &= k(z,z) - k_t^\top(z) (K_t + \lambda I)^{-1} k_t(z),
\end{align}
where $k_t(z) = [k(z, z_1), \dots, k(z, z_t)]^\top$, $K_t = [k(z_i, z_j)]_{i,j=1}^t$, and $\lambda > 0$ is a regularization parameter. Confidence bounds of the form $|f(z) - \hat{f}_t(z)| \leq \beta_t \sigma_t(z)$ hold with high probability under standard conditions.

%\paragraph{KOVI} Recall the KOVI algorithm from \cite{yang2020provably}:

%\begin{algorithm}[H]
%\caption{The KOVI Algorithm}
%\label{alg:KRVI}
%\begin{algorithmic}[1]
%\State \textbf{Input:} Regularization \( \lambda \), confidence parameter \( \beta_T(\delta) \), kernel \( k \), MDP \( M = (\Sc, \Ac, H, P, r) \)
%\For{episode \( t = 1, 2, \dots, T \)}
%    \State Observe initial state \( s_1^t \)
%    \State Set \( V^t_{H+1}(s) = 0 \) for all \( s \in \Sc \)
%   \For{step \( h = H, H-1, \dots, 1 \)}
%        \State Compute \( Q_h^t(z) \) and \( V_h^t(z) \) via kernel ridge regression
%    \EndFor
%    \For{step \( h = 1, 2, \dots, H \)}
%        \State Take action \( a_h^t \gets \arg\max_{a \in \Ac} Q_h^t(s_h^t, a) \)
%        \State Observe reward \( r_h(s_h^t, a_h^t) \) and next state \( s_{h+1}^t \)
%    \EndFor
%\EndFor
%\end{algorithmic}
%\end{algorithm}

\section{Proof of covering number}\label{app:proof_cov}
Since we only care about the dependence of $\log \mathcal N(\epsilon, B, R, G)$ on the newly introduced variable $G$, in this section we will simplify notation.
Firstly, the definition of the state-action value function class, 
\begin{equation}\nonumber
\gQ = \{Q=\hat{Q}_T(\cdot) + \beta \sigma_T(\cdot), \forall\{z_t\}_{t=1}^T \subset \gZ\}.
\end{equation}
and the notation $\gN(\epsilon)$ for its $\epsilon$-covering number.
Let us use the notation $\gN_{k,R}(\epsilon)$ for the $\epsilon$-covering number of RKHS ball $\gB_{k,R}=\{f:\|f\|_{\Hc_k}\le R\}$, $\gN_{[0,B]}(\epsilon)$ for the $\epsilon$-covering number of interval $[0,B]$ with respect to Euclidean distance, and $\gN_{k,\bm{b}}(\epsilon)$ for the $\epsilon$-covering number of class of uncertainty functions $\bm{b}_k=\{b(z)=\left(k(z,z)-k^\top_Z(z)(K_Z+\lambda^2I)^{-1}k_Z(z)\right)^{\frac{1}{2}}, |Z|\le T\}$.

Consider $Q,\bar{Q}\in \gQ $ where $Q(z)= \min\left\{Q_0(z)+\beta b(z),H-h+1\right\}$ and $\bar{Q}(z)= \min\left\{\bar{Q}_0(z)+\bar{\beta} \bar{b}(z),H-h+1\right\}$. We have
\begin{eqnarray}
    |Q(z)-\bar{Q}(z)| \le   |Q_0(z)-\bar{Q}_0(z)| + |\beta-\bar{\beta}| + B|b(z)-\bar{b}(z)|.
\end{eqnarray}

That implies 
\begin{eqnarray}\label{eq:disint}
    \gN \le \gN_{k,R}(\frac{\epsilon}{3})\gN_{[0,B]}(\frac{\epsilon}{3})\gN_{k,\bm{b}}(\frac{\epsilon}{3B}).
\end{eqnarray}

For the $\epsilon$-covering number of the $[0,B]$ interval, we simply have $\gN_{[0,B]}(\epsilon/3)\le1+3B/\epsilon$. 
In the next lemmas, we bound the $\epsilon$-covering number of the RKHS ball and the class of uncertainty functions.

\begin{lemma}[RKHS Covering Number]\label{lem:RKHScov_num}
Consider a positive definite kernel $k:\Zc\times \Zc\rightarrow \Rr$. Under Assumption~\ref{ass:eigendecay}, the $\epsilon$-covering number of $R$-ball in the RKHS satisfies
\begin{eqnarray}
    \log \gN_{k,R}(\epsilon) =\gO\left(\left(\frac{R^2}{\epsilon^2|G|^p}\right)^{\frac{1}{p-1}} \log(1+\frac{R}{\epsilon})\right).
\end{eqnarray}
\end{lemma}

\begin{lemma}[Uncertainty Class Covering Number]\label{lem:Uncer_cov_num}
Consider a positive definite kernel $k:\Zc\times \Zc\rightarrow \Rr$. Under Assumption~\ref{ass:eigendecay}, the $\epsilon$-covering number of the class of uncertainty functions satisfies 
\begin{eqnarray}
    \log \gN_{k,\bm{b}}(\epsilon) =\gO\left((\frac{1}{\epsilon^2|G|^p})^{\frac{2}{p-1}}(1+\log(\frac{1}{\epsilon}))\right)
\end{eqnarray}

\end{lemma}

Combining~\eqref{eq:disint} with Lemmas~\ref{lem:RKHScov_num} and~\ref{lem:Uncer_cov_num}, we obtain

    \begin{eqnarray}
    \log \gN(\epsilon) =\gO\left( (\frac{R^2}{\epsilon^2 |G|^p})^{\frac{1}{p-1}}(1+\log(\frac{R}{\epsilon})) + (\frac{B^2}{\epsilon^2|G|^p})^{\frac{2}{p-1}}(1+\log(\frac{B}{\epsilon}))\right),
\end{eqnarray}

that completes the proof of the theorem.
Next, we provide the proof of two lemmas above on the covering numbers of the RKHS ball and the uncertainty function class.

\begin{proof}[Proof of Lemma~\ref{lem:RKHScov_num}]

For $f$ in the RKHS, recall the following representation
\begin{eqnarray}
f = \sum_{m=1}^\infty w_m\sqrt{\lambda_m}\phi_m,
\end{eqnarray}
as well as its projection on the $D$-dimensional RKHS
\begin{eqnarray}
\Pi_D[f] = \sum_{m=1}^D w_m\sqrt{\lambda_m}\phi_m.
\end{eqnarray}
We note that
\begin{eqnarray}\nonumber
\|f-\Pi_D[f]\|_{\infty}&=&
\sum_{m=D+1}^{\infty} w_m\sqrt{\lambda_m}\phi_m\\\nonumber
&\le& C_1 \sum_{m=D+1}^{\infty} |w_m|(m\theta)^{-\frac{p}{2}}\\\nonumber
&\le& C_1 \left(\sum_{m=D+1}^{\infty} |w_m|^2\right)^{\frac{1}{2}}
\left(
\sum_{m=D+1}^{\infty }(m\theta)^{-p}
\right)^{\frac{1}{2}}
\\\nonumber
&\le& C_1 R\left(\int_{D}^{\infty}x^{-p}dx\right)^{\frac{1}{2}}\\\nonumber
&=&\frac{C_1 R\theta^{-\frac{p}{2}}}{\sqrt{p-1}}D^{\frac{-p+1}{2}}.
\end{eqnarray}

The second inequality follows form Cauchy–Schwarz inequality.

Now, let $D_0$ be the smallest $D$ such that the right hand side is bounded by $\frac{\epsilon}{2}$. There exists a constant $C_2>0$, only depending on constants $C_1$ and $p$, such that 
\begin{eqnarray}
    D_0\le C_2\left({\frac{R\theta^{-p/2}}{\epsilon}}\right)^{\frac{2}{p-1}}.
\end{eqnarray}

For a $D$-dimensional linear model, where the norm of the weights is bounded by $R$, the $\epsilon$-covering is at most $C_3 D(1+\log(\frac{R}{\epsilon})$, for some constant $C_3$~\citep[e.g., see,][]{yang2020provably}. Using an $\epsilon/2$ covering number for the space of $\Pi_D[f]$ and using the minimum number of dimensions that ensures $|f-\Pi_D[f]|\le\epsilon/2$, we conclude that
\begin{eqnarray}\nonumber
    \log\gN_{k,R}(\epsilon) &\le& C_3D_0(1+\log(\frac{R}{\epsilon}))\\\nonumber
    &\le& C_2C_3\left({\frac{R\theta^{-p/2}}{\epsilon}}\right)^{\frac{2}{p-1}}(1+\log(\frac{R}{\epsilon})),
\end{eqnarray}

that completes the proof of the lemma. 
    
\end{proof}

\begin{proof}[Proof of Lemma~\ref{lem:Uncer_cov_num}]

Let us define $\bm{b}_k^2=\{b^2:b\in\bm{b}_k\}$ and $\gN_{k,\bm{b}^2}({\epsilon})$ to be its $\epsilon$-covering number. 
We note that, for $b,\bar{b}\in\bm{b}$,
\begin{eqnarray}
   | b(z) - \bar{b}(z) | \le \sqrt{|(b(z))^2 - (\bar{b}(z))^2|}.    
\end{eqnarray}
Thus, an $\epsilon$-covering number of $\bm{b}$ is bounded by an $\epsilon^2$-covering of $\bm{b}^2$:
\begin{equation}
    \gN_{k,\bm{b}}({\epsilon}) \le \gN_{k,\bm{b}^2}({\epsilon^2}).
\end{equation}

We next bound $\gN_{k,\bm{b}^2}({\epsilon^2})$.

Using the feature space representation of the kernel, we obtain
\begin{equation}
    (b(z))^2 = \sum_{m=1}^\infty \gamma_m\lambda_m\phi^2_m(z),
\end{equation}
for some $\gamma_m\in[0,1]$. Based on the GP interpretation of the model, $\gamma_m$ can be understood as the posterior variances of the weights. Let $D_0$ be the smallest $D$ such that $\sum_{m=D+1}^\infty \lambda_m\phi^2_m(z)\le \epsilon^2/2$. Similar to the previous lemma, we can show that, for some constant $C_4$, only depending on constants $C_1$ and $p$, 
\begin{equation}\label{eq:d0}
    D_0\le C_4\left(\frac{1}{\theta^p\epsilon^2}\right)^{\frac{1}{p-1}}.
\end{equation}

For $\sum_{m=1}^{D_0} \gamma_m\lambda_m\phi^2_m(z)$ on a finite $D_0$-dimensional spectrum, as shown in Lemma~$D.3$ of \cite{yang2020provably}, an $\epsilon^2/2$ covering number scales with $D_0^2$. Specifically, an $\epsilon^2/2$ covering number of the space of $\sum_{m=1}^{D_0} \gamma_m\lambda_m\phi^2_m(z)$ is bounded by 
\begin{equation}\label{eq:cp}
C_5D_0^2(1+\log(\frac{1}{\epsilon})).
\end{equation}

Combining Equations~\eqref{eq:d0} and~\eqref{eq:cp}, we obtain

\begin{eqnarray}\nonumber
    \gN_{k,\bm{b}^2}({\epsilon^2}) &\le& C_5D_0^2(1+\log(\frac{1}{\epsilon}))\\\nonumber
    &\le& C_5C_4^2\left(\frac{1}{\theta^p\epsilon^2}\right)^{\frac{2}{p-1}}(1+\log(\frac{1}{\epsilon})),
\end{eqnarray}
that completes the proof of the lemma. 
\end{proof}

\begin{proof}[Proof of Corollary~\ref{cor:regret}]
From the definition of $k^*$:
$$
\kappa^* = \max \left( \xi^*, \frac{2d + p + 1}{(d + p) \cdot [p - 1]}, \frac{2}{p - 3} \right),
$$
$$
\xi^* = \frac{d + 1}{2(p + d)}.
$$
Then:
$$\gN(\epsilon(T), B_T,R_T,G) = \tilde \gO\left(T^{\frac{4p^2+p-1}{2p^2-p}}\left(\frac{1}{|G|}\right)^{\frac{2p-1}{p-1}} + T^{k^*+\frac{4p}{p-1}} \left(\frac{1}{|G|}\right)^{\frac{2p}{p-1}} \right)$$
which is dominated by the $\Gamma_{k_G}(T)$ term in its dependence on $|G|$, hence the result.
\end{proof}

\section{Experimental details}\label{app:exp}
Here, we outline the procedure of generating $r$ and $P$ test functions from the RKHS in the synthetic setting. We also detail the hyperparameters used in both the synthetic and Frozen Lake experiments, along with the computational resources required.
\subsection{Synthetic Setting}\label{app:synthetic}
The reward function $r$ and transition probability $P$ are chosen as arbitrary functions from the RKHS of the invariant kernel $k_G$. To generate $r$, we draw a Gaussian Process (GP) sample on a subset of the domain $\Zc$. This subset consists of evenly spaced points on a $5 \times 5$ grid covering the range $[-1, 1]$ in both dimensions. Kernel ridge regression is then fitted to these samples, and the predictions are scaled to the $[0,1]$ range to produce $r$ (see Figure~\ref{fig:R_RBF}). To construct $P(s'|s,a)$, we similarly draw a GP sample on a subset of the domain $\Zc\times \Sc$, fit kernel ridge regression to these samples, and then shift and rescale for each $z$ to yield a valid conditional probability distribution $P(\cdot|z)$. This is a common approach to create functions belonging to an RKHS ~\citep[e.g., see,][]{chowdhury2017kernelized}. We use RBF as the base kernel of $k_G$, with length scale $=0.1$  and $\lambda=0.01$.

For the KRVI algorithm, we use a length scale of $1$, $\lambda = e^{-10}$, and a confidence interval width multiplier $\beta = 0.1$ for both the standard RBF and the invariant kernel. Results are averaged over 20 random seeds. Kernel ridge regression is implemented using the Scikit-Learn library~\citep{pedregosa2011scikit}, which provides robust and efficient tools for kernel-based machine learning models. Simulations are run on a computing cluster with 512 GiB of RAM and two Intel(R) Xeon(R) Gold 6248 CPUs running at 2.5 GHz.
\begin{figure}[h]
    \centering
    \includegraphics[width=0.45\textwidth]{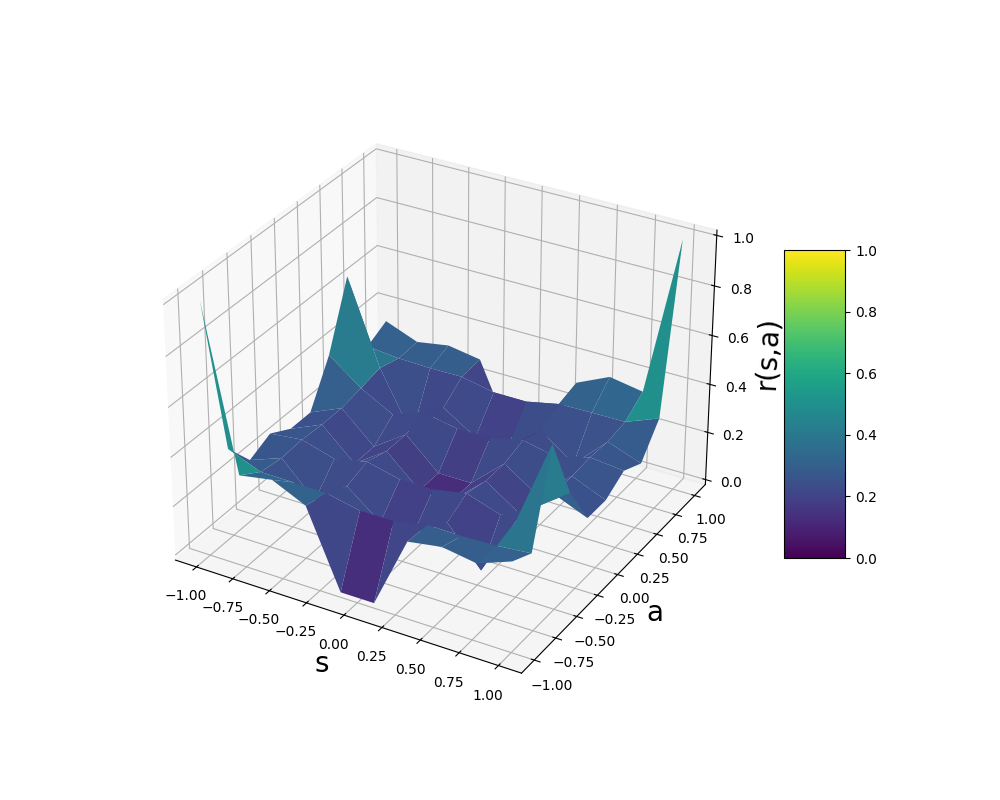}
    \caption{Reward function $r(s,a)$ generated by kernel ridge regression using an invariant kernel with lengthscale $=0.1$ and $\lambda= 0.01$}
    \label{fig:R_RBF}
\end{figure}
\subsection{Frozen Lake Experiments}\label{app:frozen_lake}
%\begin{figure}[h]
 %   \centering
  %  \includegraphics[width=0.3\textwidth]{NeurIPS_2025/figures/Frozen_Lake.png}
   % \caption{Frozen Lake}
    %\label{fig:Frozen Lake}
%\end{figure}
For the Fixed Layout Frozen Lake  and Random Layout Frozen Lake experiments, we use the BoTorch library~\citep{balandat2020botorch}, specifically the SingleTaskGP model. Simulations are conducted on a node with 376.2 GiB of RAM and an Intel(R) Xeon(R) Gold 5118 CPU running at 2.30 GHz. We use the code for invariant kernels made publicly available by~\citep{brown2024sample}.

We tune the confidence interval width multiplier $\beta$ for each setting (Fixed Layout and Random Layout) for both the invariant and standard RBF kernels by selecting the best-performing value from the grid $\{0.01, 0.1, 0.5, 1\}$. The kernel length scale is also tuned over the set $\{0.1, 0.5, 1\}$. The noise regularization parameter $\lambda$ is initialized to $0.1$ and optimized during training by maximizing the marginal log-likelihood, as handled by the BoTorch library.

\textcolor{black}{To assess sensitivity to $\beta$, Table~\ref{tab:ablation} reports cumulative returns for different values of $\beta$ on the Frozen Lake (Fixed) environment for both the standard RBF and invariant kernels, while keeping all other hyperparameters fixed to their selected values.}

\begin{table}[h]
\centering
\caption{Cumulative returns for different $\beta$ values on Frozen Lake (Fixed), with standard RBF and invariant kernels.}
\label{tab:ablation}

\begin{tabular}{ccccccccc}
\toprule
 & \multicolumn{4}{c}{RBF kernel} & \multicolumn{4}{c}{Invariant kernel} \\
\cmidrule(lr){2-5} \cmidrule(lr){6-9}
Episodes & $\beta = 0.01$ & $\beta = 0.1$ & $\beta = 0.5$ & $\beta = 1.0$ 
         & $\beta = 0.01$ & $\beta = 0.1$ & $\beta = 0.5$ & $\beta = 1.0$ \\
\midrule
250  & $\sim$22  & $\sim$22  & $\sim$22  & $\sim$22 
     & $\sim$42  & $\sim$43  & $\sim$11  & $\sim$8 \\
500  & $\sim$77  & $\sim$47  & $\sim$52  & $\sim$53 
     & $\sim$138 & $\sim$138 & $\sim$24  & $\sim$23 \\
1000 & $\sim$318 & $\sim$174 & $\sim$85  & $\sim$83 
     & $\sim$500 & $\sim$371 & $\sim$41  & $\sim$46 \\
1500 & $\sim$779 & $\sim$381 & $\sim$129 & $\sim$123 
     & $\sim$964 & $\sim$626 & $\sim$62  & $\sim$86 \\
\bottomrule
\end{tabular}
\end{table}

These results indicate that smaller values of $\beta$ tend to perform better in practice, as larger values can lead to excessive exploration. This trend is consistent across both RBF and invariant kernels and motivates the selection of $\beta = 0.01$ in our experiments. We observe similar qualitative behavior across settings.

For \textbf{FrozenLake (Fixed)}, the selected hyperparameters are:  
\begin{itemize}
    \item  RBF kernel: $\beta = 0.01$, length scale $= 0.1$, $\lambda$ initialized to $0.1$ (optimized during training). 
    \item Invariant kernel: $\beta = 0.01$, length scale $= 0.5$, $\lambda$ initialized to $0.1$ (optimized during training).  
\end{itemize}

For \textbf{FrozenLake (Random)}, the selected hyperparameters are:  
\begin{itemize}
    \item RBF kernel: $\beta = 0.01$, length scale $= 1$, $\lambda$ initialized to $0.1$ (optimized during training).
    \item Invariant kernel: $\beta = 0.01$, length scale $= 0.5$, $\lambda$ initialized to $0.1$ (optimized during training). 
\end{itemize}

\subsection{Additional Experiments: KOVI (with and without symmetry) vs. DQN (with and without symmetry)}\label{appendix:dqn_v_kovi}
To evaluate the performance of our kernel-based approach (KOVI with a standard RBF kernel) and its symmetry-aware variant (KOVI with an invariant kernel) against neural network–based methods, we compared them with DQN and its symmetrized counterpart (SymDQN) (see Figure~\ref{fig:DQN_KOVI}). This comparison highlights the practical advantages of our kernel-based methods and provides insight into how they perform relative to deep RL approaches that explicitly incorporate symmetry (e.g., equivariant networks). Both DQN and SymDQN were implemented with the same base network architecture, with SymDQN using an equivariant policy network within the Stable-Baselines3 API, tailored to the rotational symmetries of the FrozenLake (Fixed) environment.  

The results show that KOVI is substantially more sample-efficient than DQN, converging to the highest return with significantly fewer environment timesteps. Likewise, KOVI with an invariant kernel converges faster than SymDQN.  Moreover, in both the kernel-based and neural approaches, the symmetry-aware variants consistently outperform their non-symmetric counterparts, highlighting the benefits of incorporating symmetry. Overall, kernel-based methods (with and without symmetry) achieve higher sample efficiency than their neural counterparts. These findings demonstrate that kernel methods can be more effective than neural networks, particularly in low-data regimes. While neural networks are known to be more scalable, they generally require large amounts of data to perform well. In contrast, kernel-based approaches tend to train faster and achieve superior performance when data are limited. This makes kernel methods especially valuable in structured, data-scarce environments where prior knowledge—such as symmetry—can be effectively exploited.  

For completeness, DQN was run with light hyperparameter tuning over a grid of discrete values for the following parameters: \texttt{learning\_starts} $\in \{0, 1000\}$, \texttt{train\_freq} $\in \{1, 100\}$, \texttt{batch\_size} $\in \{128, 256\}$, \texttt{exploration\_fraction} $\in \{0.01, 0.05\}$, \texttt{target\_update\_interval} $\in \{500, 1000\}$, and \texttt{learning\_rate} $\in \{1\mathrm{e}{-4}, 5\mathrm{e}{-5}\}$. The final configuration, which performed best, was:  
 \texttt{learning\_starts = 1000}, \texttt{train\_freq = 1}, \texttt{gradient\_steps = 1}, \texttt{batch\_size = 256}, \texttt{exploration\_fraction = 0.05}, \texttt{exploration\_initial\_eps = 1}, \texttt{exploration\_final\_eps = 0.05}, \texttt{target\_update\_interval = 500}, \texttt{learning\_rate = 0.0001}, \texttt{buffer\_size = 100000}.  

% \begin{figure}[h] 
%     \centering
%     \includegraphics[width=0.5\textwidth]{AISTATS2026PaperPack/figures/DQN_KOVI_edited.png}
%     \caption{Comparison of KOVI with a standard RBF kernel, KOVI with an invariant kernel, DQN, and SymDQN on the FrozenLake (Fixed) environment. Average episodic return is plotted against the number of environment timesteps. Results are averaged over 20 random seeds. The shaded area represents the standard error.}
%     \label{fig:DQN_KOVI_}
% \end{figure}

\subsection{SynPl experiments}
We reuse most of the settings from the Frozen Lake experiments, but make relevant modifications to the environment to encode actions and states correctly. Since both actions and states represent positions on an 8x8 grid, we can produce a uniform representation for states and actions. Each position vector is a half integer in $[-4, 4]^2$ corresponding to the centers of squares of an 8x8 centered lattice at (0,0). Hyperparameter tuning is performed for $\beta$, the starting length scale and the starting noise regularization. We run experiments either with these values fixed or with optimizing these values in BoTorch and warm-starting the optimization in subsequent episodes. We perform hyperparameter tuning separately for the RBF baseline and for the invariant kernel, averaging performance over 5 seeds and then selecting the best configuration of hyperparameters. 
We then rerun the best configurations for 3 seeds for a total of 4000 episodes. Across both configurations, BoTorch optimization of the initial parameters provides improvement, and the optimal values found are $\beta=0.1$ for the invariant kernel, $\beta=0.05$ for the RBF kernel. For both kernels the optimal length scale was 1, and optimal noise regularization was $10^{-6}$.

\end{document}